\documentclass{article}

\newif\ifarxiv
\arxivtrue

\ifarxiv 
\usepackage[final]{neurips_2022}
\let\cite\citep
\else 
\usepackage{neurips_2022}
\fi

\usepackage{microtype}
\usepackage{graphicx}
\usepackage{floatrow}
\newfloatcommand{capbtabbox}{table}[][\FBwidth]
\usepackage{blindtext}
\usepackage{subfigure}
\usepackage{booktabs} 

\usepackage{amsmath,amsfonts,bm}

\def\eqref#1{equation~\ref{#1}}

\def\1{\bm{1}}

\def\rr{{\textnormal{r}}}

\def\vone{{\bm{1}}}

\def\vtheta{{\bm{\theta}}}
\def\va{{\bm{a}}}
\def\vb{{\bm{b}}}
\def\vc{{\bm{c}}}

\def\ve{{\bm{e}}}
\def\vf{{\bm{f}}}
\def\vg{{\bm{g}}}

\def\vu{{\bm{u}}}
\def\vv{{\bm{v}}}
\def\vw{{\bm{w}}}
\def\vx{{\bm{x}}}
\def\vy{{\bm{y}}}
\def\vz{{\bm{z}}}

\DeclareMathAlphabet{\mathsfit}{\encodingdefault}{\sfdefault}{m}{sl}
\SetMathAlphabet{\mathsfit}{bold}{\encodingdefault}{\sfdefault}{bx}{n}

\DeclareMathOperator*{\argmax}{arg\,max}

\usepackage{amsthm}
\usepackage{thmtools} 
\usepackage{thm-restate}
\usepackage{xcolor}
\usepackage[export]{adjustbox}

\usepackage{hyperref}

\usepackage{multirow}

\usepackage[utf8]{inputenc}
\usepackage{color}

\def\rr{\mathbb{R}}

\def\cR{\mathcal{R}}
\def\aug{\mathrm{aug}}
\def\diag{\mathrm{diag}}
\def\dd{\mathrm{d}}

\def\bn{\mathrm{n}}
\def\sg{\mathrm{sg}}

\def\tr{\mathrm{tr}}
\def\vec{\mathrm{vec}}

\def\ee{\mathbb{E}}
\def\cov{\mathbb{V}}
\def\con{\mathbb{C}}
\def\pr{\mathbb{P}}

\def\cE{\mathcal{E}}
\def\cA{\mathcal{A}}
\def\cW{\mathcal{W}}
\def\cJ{\mathcal{J}}
\def\cN{\mathcal{N}}
\def\cL{\mathcal{L}}
\def\vtheta{\boldsymbol{\theta}}

\def\method{$\alpha$-CL}
\def\boldmethod{$\boldsymbol{\alpha}$\textbf{-CL}}

\newtheorem{definition}{\textbf{Definition}}
\newtheorem{lemma}{\textbf{Lemma}}
\newtheorem{assumption}{\textbf{Assumption}}
\newtheorem{theorem}{\textbf{Theorem}}

\newcommand{\yuandong}[1]{\textcolor{red}{TODO: #1}}

\begin{document}

\title{Understanding Deep Contrastive Learning via Coordinate-wise Optimization}

\vspace{-0.2in}



\author{Yuandong Tian \\ 
Meta AI (FAIR) \\
\texttt{yuandong@meta.com}}
\maketitle




\vspace{-0.2in}
\begin{abstract}
We show that Contrastive Learning (CL) under a broad family of loss functions (including InfoNCE) has a unified formulation of coordinate-wise optimization on the network parameter $\vtheta$ and pairwise importance $\alpha$, where the \emph{max player} $\vtheta$ learns representation for contrastiveness, and the \emph{min player} $\alpha$ puts more weights on pairs of distinct samples that share similar representations. The resulting formulation, called \boldmethod{}, unifies not only various existing contrastive losses, which differ by how sample-pair importance $\alpha$ is constructed, but also is able to extrapolate to give novel contrastive losses beyond popular ones, opening a new avenue of contrastive loss design. These novel losses yield comparable (or better) performance on CIFAR10, STL-10 and CIFAR-100 than classic InfoNCE. Furthermore, we also analyze the max player in detail: we prove that with fixed $\alpha$, max player is equivalent to Principal Component Analysis (PCA) for deep linear network, and almost all local minima are global and rank-1, recovering optimal PCA solutions. Finally, we extend our analysis on max player to 2-layer ReLU networks, showing that its fixed points can have higher ranks. Codes are available~\footnote{\url{https://github.com/facebookresearch/luckmatters/tree/main/ssl/real-dataset}}. 
\end{abstract}

\vspace{-0.1in}
\section{Introduction}
\vspace{-0.1in}
While contrastive self-supervised learning has been shown to learn good features~\cite{chen2020simclr,He2020MomentumCF,oord2018representation} and in many cases, comparable with features learned from supervised learning, it remains an open problem what features it learns, in particular when deep nonlinear networks are used. Theory on this is quite sparse, mostly focusing on loss function~\cite{arora2019theoretical} and treating the networks as a black-box function approximator.  

In this paper, we present a novel perspective of contrastive learning (CL) for a broad family of contrastive loss functions $\cL(\vtheta)$: minimizing $\cL(\vtheta)$ corresponds to a \emph{coordinate-wise optimization} procedure on an objective $\cE_\alpha(\vtheta) - \cR(\alpha)$ with respect to network parameter $\vtheta$ and \emph{pairwise importance} $\alpha$ on batch samples, where $\cE_\alpha(\vtheta)$ is an energy function and $\cR(\alpha)$ is a regularizer, both associated with the original contrastive loss $\cL$. In this view, the \emph{max player} $\vtheta$ learns a representation to maximize the contrastiveness of different samples and keep different augmentation view of the same sample similar, while the \emph{min player} $\alpha$ puts more weights on pairs of different samples that appear similar in the representation space, subject to regularization. Empirically, this formulation, named P\underline{a}ir-weighed \underline{C}ontrastive \underline{L}earning ({\boldmethod{}}), when coupled with various regularization terms, yields novel contrastive losses that show comparable (or better) performance in CIFAR10~\cite{krizhevsky2009learning} and STL-10~\cite{coates2011analysis}.  

We then focus on the behavior of the max player who does \emph{representation learning} via maximizing the energy function $\cE_\alpha(\vtheta)$. When the underlying network is deep linear, we show that $\max_{\vtheta}\cE_\alpha(\vtheta)$ is the loss function (under re-parameterization) of Principal Component Analysis (PCA)~\cite{wold1987principal}, a century-old unsupervised dimension reduction method. To further show they are equivalent, we prove that the nonlinear training dynamics of CL with a linear multi-layer feedforward network (MLP) enjoys nice properties: with proper weight normalization, almost all its local optima are global, achieving optimal PCA objective, and are rank-1. The only difference here is that the data augmentation provides negative eigen-directions to avoid.  

Furthermore, we extend our analysis to 2-layer ReLU network, to explore the difference between the rank-1 PCA solution and the solution learned by a nonlinear network. Assuming the data follow an orthogonal mixture model, the 2-layer ReLU networks enjoy similar dynamics as the linear one, except for a special \emph{sticky weight rule} that keeps the low-layer weights to be non-negative and stays zero when touching zero. In the case of one hidden node, we prove that the solution in ReLU always picks a single mode from the mixtures. In the case of multiple hidden nodes, the resulting solution is not necessarily rank-1. 

\vspace{-0.05in}
\section{Related Work}
\vspace{-0.05in}
\textbf{Contrastive learning}. While many contrastive learning techniques (e.g., SimCLR~\cite{chen2020simclr}, MoCo~\cite{He2020MomentumCF}, PIRL~\cite{misra2020pirl}, SwAV~\cite{caron2020swav}, DeepCluster~\cite{Caron2018DeepCF}, Barlow Twins~\cite{zbontar2021barlow}, InstDis~\cite{wu2018unsupervised}, etc) have been proposed empirically and able to learn good representations for downstream tasks, theoretical study is relatively sparse, mostly focusing on loss function itself~\cite{tian2020makes,haochen2021provable,arora2019theoretical}, e.g., the relationship of loss functions with mutual information (MI). To our knowledge, there is no analysis that combines the property of neural network and that of loss functions.  

\textbf{Theoretical analysis of deep networks}. Many works focus on analysis of deep linear networks in supervised setting, where label is given. \cite{baldi1989neural,zhou2018critical,kawaguchi2016deep} analyze the critical points of linear networks. \cite{saxe2013exact,arora2018optimization} also analyze the training dynamics. On the other hand, analyzing nonlinear networks has been a difficult task. Existing works mostly lie in supervised learning, e.g., teacher-student setting~\cite{tian2019student,allen2018learning}, landscape~\cite{safran2018spurious}. For contrastive learning, recent work~\cite{wen2021toward} analyzes the dynamics of 1-layer ReLU networks with a specific weight structure, and~\cite{jing2021understanding} analyzes the collapsing behaviors in 2-layer linear network for CL. To our best knowledge, we are not aware of such analysis on deep networks ($> 2$ layers, linear or nonlinear) in the context of CL.  

\textbf{Connection between Principal Component Analysis (PCA) and Self-supervised Learning}. ~\cite{lee2021predicting} establishes the statistical connection between non-linear Canonical Component Analysis (CCA) and SimSiam~\cite{chen2020simsiam} for any zero-mean encoder, without considering the aspect of training dynamics. In contrast, we reformulate contrastive learning as coordinate-wise optimization procedure with min/max players, in which the max player is a reparameterization of PCA optimized with gradient descent, and analyze its training dynamics in the presence of specific neural architectures. 

\begin{figure}
    \includegraphics[width=.32\textwidth]{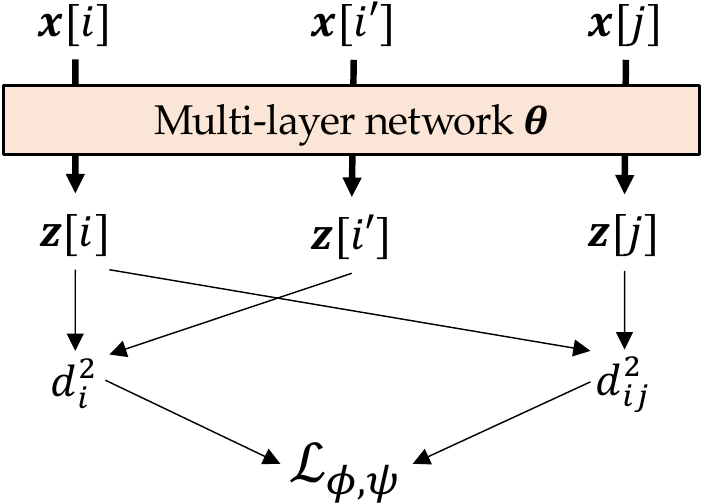}
    \hfill 
    \footnotesize
    \setlength\tabcolsep{2pt}
    \begin{tabular}[b]{l|l|l|}
    Contrastive Loss & $\phi(x)$ & $\psi(x)$ \\
    \hline 
    \hline
    InfoNCE~{\tiny\cite{oord2018representation}} & $\tau\log(\epsilon + x)$ & $e^{x/\tau}$ \\
    MINE~{\tiny\cite{belghazi2018mutual}} & $\log(x)$ & $e^x$ \\
    Triplet~{\tiny\cite{schroff2015facenet}} & $x$ & $[x + \epsilon]_+$ \\ 
    Soft Triplet~{\tiny\cite{tian2020understanding}}  & $\tau\log(1 + x)$ & $e^{x/\tau + \epsilon}$ \\
    N+1 Tuplet~{\tiny\cite{sohn2016improved}}      & $\log(1+x)$ & $e^x$ \\
    Lifted Structured~{\tiny\cite{oh2016deep}} & $[\log(x)]^2_+$ & $e^{x + \epsilon}$ \\
    Modified Triplet~{\tiny Eqn.~10~\cite{coria2020comparison}} & $x$ & $\mathrm{sigmoid}(c x)$ \\ 
    Triplet Contrastive~{\tiny Eqn.~2~\cite{ji2021power}} & linear & linear \\
    \end{tabular}
    \caption{\small Problem Setting. \textbf{Left}: Data points ($i$-th sample $\vx[i]$ and its augmented version $\vx[i']$, $j$-th sample $\vx[j]$) are sent to networks with weights $\vtheta$, to yield outputs $\vz[i]$, $\vz[i']$ and $\vz[j]$. From the outputs $\vz$, we compute pairwise squared distance $d^2_{ij}$ between $\vz[i]$ and $\vz[j]$ and intra-class squared distance $d^2_i$ between $\vz[i]$ and $\vz[i']$ for contrastive learning with a general family of contrastive loss $\cL_{\phi, \psi}$ (Eqn.~\ref{eq:general-loss}). \textbf{Right}: Different existing loss functions corresponds to different monotonous functions $\phi$ and $\psi$. Here $[x]_+ := \max(x, 0)$.}
    \label{tab:loss-funcs}
\end{figure}

\vspace{-0.05in}
\section{Contrastive Learning as Coordinate-wise Optimization} 
\vspace{-0.05in}
\label{sec:general-game-framework}
\textbf{Notation}. Suppose we have $N$ pairs of samples $\{\vx[i]\}_{i=1}^N$ and $\{\vx[i']\}_{i=1}^N$. Both $\vx[i]$ and $\vx[i']$ are augmented samples from sample $i$ and $\vx$ represents the input batch. These samples are sent to neural networks and $\vz[i]$ and $\vz[i']$ are their outputs. The goal of contrastive learning (CL) is to find the representation to maximize the squared distance $d^2_{ij} := \|\vz[i]-\vz[j]\|_2^2/2$ between distinct samples $i$ and $j$, and minimize the squared distance $d^2_i := \|\vz[i]-\vz[i']\|_2^2/2$ between different data augmentations $\vx[i]$ and $\vx[i']$ of the same sample $i$.

\vspace{-0.05in}
\subsection{A general family of contrastive loss}
\vspace{-0.05in}
We consider minimizing a general family of loss functions $\cL_{\phi,\psi}$, where $\phi$ and $\psi$ are monotonously increasing and differentiable scalar functions (define $\xi_i := \sum_{j\neq i} \psi(d^2_i - d^2_{ij})$ for notation brevity): 
\begin{equation}
    \min_{\vtheta} \cL_{\phi,\psi}(\vtheta) := \sum_{i=1}^N \phi(\xi_i) = \sum_{i=1}^N \phi\left(\sum_{j\neq i} \psi(d^2_i - d^2_{ij})\right) \label{eq:general-loss}
\end{equation}
Both $i$ and $j$ run from $1$ to $N$. With different $\phi$ and $\psi$, Eqn.~\ref{eq:general-loss} covers many loss functions (Tbl.~\ref{tab:loss-funcs}). In particular, setting $\phi(x) = \tau\log(\epsilon + x)$ and $\psi(x) = \exp(x/\tau)$ gives a generalized version of InfoNCE loss~\cite{oord2018representation}:
\begin{equation}
    \!\!\cL_{nce}\!:=\!-\tau \sum_{i=1}^N\log\frac{\exp(-d^2_i/\tau)}{\epsilon\exp(-d^2_i/\tau)\!+\!\sum_{j\neq i} \exp(-d^2_{ij}/\tau)} = \tau \sum_{i=1}^N\log\left(\epsilon\!+\!\sum_{j\neq i} e^{\frac{d_i^2-d^2_{ij}}{\tau}}\right)
\end{equation}
where $\epsilon > 0$ is some constant not related to $\vz[i]$ and $\vz[i']$. $\epsilon = 1$ has been used in many works~\cite{He2020MomentumCF,tian2020contrastive}. Setting $\epsilon = 0$ yields SimCLR setting~\cite{chen2020simclr} where the denominator doesn't contains $\exp(-d^2_i/\tau)$. This is also used in~\cite{yeh2021decoupled}. 

\def\vtheta{\boldsymbol{\theta}}

\subsection{The other side of gradient descent of contrastive loss}
To minimize $\cL_{\phi,\psi}$, gradient descent follows its negative gradient direction. As a first discovery of this work, it turns out that the gradient descent of the loss function $\cL$ is the \emph{gradient ascent} direction of another energy function $\cE_\alpha$:

\begin{restatable}{theorem}{gradientmatch}
\label{thm:general-E}
For any differential mapping $\vz = \vz(\vx; \vtheta)$, gradient descent of $\cL_{\phi,\psi}$ is equivalent to \underline{\textbf{gradient ascent}} of the objective $\cE_\alpha(\vtheta) := \frac{1}{2}\tr(\con_\alpha[\vz(\vtheta), \vz(\vtheta)])$:
\begin{equation}
    \frac{\partial \cL_{\phi, \psi}}{\partial \vtheta} = -\frac{\partial \cE_{\alpha}}{\partial \vtheta}\Big|_{\alpha = \alpha(\vtheta)} \label{eq:eq-derivative}
\end{equation}
Here the \emph{pairwise importance} $\alpha = \alpha(\vtheta) := \{\alpha_{ij}(\vtheta)\}$ is a function of input batch $\vx$, defined as:
\begin{equation}
     \alpha_{ij}(\vtheta) := \phi'(\xi_i)\psi'(d^2_i - d^2_{ij}) \ge 0
     \label{eq:alpha}
\end{equation}
where $\phi',\psi'\ge 0$ are derivatives of $\phi,\psi$. The \emph{contrastive covariance} $\con_\alpha[\cdot, \cdot]$ is defined as:
\begin{equation}
  \con_\alpha[\va,\vb] := \sum_{i=1}^N\sum_{j\neq i} \alpha_{ij} (\va[i] - \va[j])(\vb[i]-\vb[j])^\top - \sum_{i=1}^N \left(\sum_{j\neq i}\alpha_{ij}\right) (\va[i] - \va[i'])(\vb[i]-\vb[i'])^\top \label{eq:contrastive-covariance}
\end{equation} 
That is, \underline{\textbf{minimizing}} the loss function $\cL_{\phi,\psi}(\vtheta)$ can be regarded as \underline{\textbf{maximizing}} the energy function $\cE_{\alpha=\sg(\alpha(\vtheta))}(\vtheta)$ with respect to $\vtheta$. Here $\sg(\cdot)$ means stop-gradient, i.e., the gradient of $\vtheta$ is not backpropagated into $\alpha(\vtheta)$.
\end{restatable}

\def\primal{\mathrm{primal}}
\def\dual{\mathrm{dual}}
\def\game{\mathrm{g}}
\def\cP{\mathcal{P}}
Please check Supplementary Materials (SM) for all proofs. From the definition of energy $\cE_\alpha(\vtheta)$, it is clear that $\alpha_{ij}$ determines the importance of each sample pair $\vx[i]$ and $\vx[j]$. For $(i, j)$-pair that ``deserves attention'', $\alpha_{ij}$ is large so that it plays a large role in the contrastive covariance term. In particular, for InfoNCE loss with $\epsilon = 0$, the pairwise importance $\alpha$ takes the following form:
\begin{equation}
 \alpha_{ij} = \frac{\exp(-d^2_{ij}/\tau)}{\sum_{j\neq i} \exp(-d^2_{ij}/\tau)} > 0 \label{eq:infonce-alpha-beta}
\end{equation}
which means that InfoNCE focuses on $(i,j)$-pair with small squared distance $d^2_{ij}$. If both $\phi$ and $\psi$ are linear, then $\alpha_{ij} = \mathrm{const}$ and $\cL$ is a simple subtraction of positive/negative squared distances. 

From Thm.~\ref{thm:general-E}, an important observation is that when propagating gradient w.r.t. $\vtheta$ using the objective $\cE_\alpha(\vtheta)$ during the backward pass, the gradient does not propagate into $\alpha(\vtheta)$, even if $\alpha(\vtheta)$ is a function of $\vtheta$ in the forward pass. In fact, in Sec.~\ref{sec:different-loss-formulation} we show that propagating gradient through $\alpha(\vtheta)$ yields worse empirical performance. This suggests that $\alpha$ should be treated as an \emph{independent} variable when optimizing $\vtheta$. It turns out that if $\psi(x)$ is an exponential function (as in most cases of Tbl.~\ref{eq:general-loss}), this is indeed true and $\alpha$ can be determined by a separate optimization procedure:
\begin{restatable}{theorem}{infoncealpha}
\label{thm:infonce-alpha}
If $\psi(x) = e^{x/\tau}$, then the corresponding pairwise importance $\alpha$ (Eqn.~\ref{eq:alpha}) is the solution to the minimization problem: 
\begin{equation}
 \alpha(\vtheta) = \arg\min_{\alpha\in \cA} \cE_\alpha(\vtheta) - \cR(\alpha),\quad\quad \cA := \left\{\alpha:\quad \forall i,\ \ \sum_{j\neq i} \alpha_{ij} = \tau^{-1} \xi_i\phi'(\xi_i), \ \ \alpha_{ij} \ge 0\right\} 
\end{equation}
Here the regularization $\cR(\alpha) = \cR_H(\alpha) := \tau\sum_{i=1}^N H(\alpha_{i\cdot}) = -\tau\sum_{i=1}^N \sum_{j\neq i} \alpha_{ij} \log \alpha_{ij}$.
\end{restatable}

For InfoNCE, the feasible set $\cA$ becomes $\{\alpha: \alpha \ge 0, \sum_{j\neq i} \alpha_{ij} = \xi_i / (\xi_i + \epsilon)\}$. This means that if $i$-th sample is already well-separated (small intra-augmentation distance $d_i$ and large inter-augmentation distance $d_{ij}$), then $\xi_i$ is small, the summation of weights $\sum_{j\neq i}\alpha_{ij}$ associated with sample $i$ is also small and such a sample is overall discounted. Setting $\epsilon=0$ reduces to sample-agnostic constraint (i.e., $\sum_{j\neq i}\alpha_{ij} = 1$).

Thm.~\ref{thm:infonce-alpha} leads to a novel perspective of \emph{coordinate-wise optimization} for Contrastive Learning (CL): 

\begin{restatable}[Contrastive Learning as Coordinate-wise Optimization]{corollary}{minimaxeq}
\label{co:minimax}
If $\psi(x) = e^{x/\tau}$, minimizing $\cL_{\phi, \psi}$ is equivalent to the following iterative procedure: 
\begin{subequations}\label{eq:minmax-opt}
\begin{align}
    \text{(Min-player $\alpha$)} \quad\quad\quad\quad\quad \alpha_t &= \arg\min_{\alpha \in \cA} \cE_\alpha(\vtheta_t) - \cR(\alpha) \label{eq:min-player} \\
    \text{(Max-player $\vtheta$)} \quad\quad\quad\quad\quad \vtheta_{t+1} &= \vtheta_t + \eta \nabla_{\vtheta} \cE_{\alpha_t}(\vtheta) \label{eq:max-player}
\end{align}
\end{subequations}
\end{restatable}
Intuitively, the max player $\vtheta$ (Eqn.~\ref{eq:max-player}) performs one-step gradient ascent for the objective $\cE_\alpha(\vtheta) - \cR(\alpha)$, \emph{learns a representation} to maximize the distance of different samples and minimize the distance of the same sample with different augmentations (as suggested by $\con_\alpha[\vz, \vz]$). On the other hand, the ``min player'' $\alpha$ (Eqn.~\ref{eq:min-player}) finds optimal $\alpha$ analytically, assigning high weights on confusing pairs for ``max player'' to solve. 

\textbf{Relation to max-min formulation}. While Corollary~\ref{co:minimax} looks very similar to max-min formulation, important differences exist. Different from traditional max-min formulation, in Corollary~\ref{co:minimax} there is asymmetry between $\vtheta$ and $\alpha$. First, $\vtheta$ only follows one step update along gradient ascent direction of $\max_{\vtheta} \cE_\alpha(\vtheta)$, while $\alpha$ is solved analytically. Second, due to the stop-gradient operator, the gradient of $\vtheta$ contains no knowledge on how $\vtheta$ changes $\alpha$. This prevents $\vtheta$ from adapting to $\alpha$'s response on changing $\vtheta$. Both give advantages to min-player $\alpha$ to find the confusing sample pairs more effectively.  

\textbf{Relation to hard-negative samples}. While many previous works~\cite{kalantidis2020hard,robinson2020contrastive} focus on seeking and putting more weights on hard samples, Corollary~\ref{co:minimax} shows that contrastive losses already have such mechanism at the batch level, focusing on ``hard-negative pairs'' beyond hard-negative samples.  

From this formulation, different pairwise importance $\alpha$ corresponds to different loss functions within the loss family specified by Eqn.~\ref{eq:general-loss}, and choosing among this family (i.e., different $\phi$ and $\psi$) can be regarded as choosing different $\alpha$ when optimizing the \emph{same} objective $\cE_\alpha(\vtheta)$. Based on this observation, we now propose the following training framework called \boldmethod{}:
\begin{definition}[P\underline{a}ir-weighed \underline{C}ontrastive \underline{L}earning (\method{})]
\label{def:method}
Optimize $\vtheta$ by gradient ascent: $\vtheta_{t+1} = \vtheta_t + \eta \nabla_{\vtheta} \cE_{\sg(\alpha_t)}(\vtheta)$, with the energy  $\cE_\alpha(\vtheta)$ defined in Thm.~\ref{thm:general-E} and pairwise importance $\alpha_t = \alpha(\vtheta_t)$. 
\end{definition}
In \method{}, choosing $\alpha$ can be achieved by either implicitly specifying a regularizer $\cR(\alpha)$ and solve Eqn.~\ref{eq:min-player}, or by a direct mapping $\alpha=\alpha(\vtheta)$ without any optimization. This opens a novel revenue for CL loss design. Initial experiments (Sec.~\ref{sec:different-loss-formulation}) show that \method{} gives comparable (or even better) downstream performance in CIFAR10 and STL-10, compared to vanilla InfoNCE loss. 

\vspace{-0.05in}
\section{Representation Learning in Deep Linear CL is PCA}
\label{sec:linear-is-pca}
\vspace{-0.05in}
In Corollary~\ref{co:minimax}, optimizing over $\alpha$ is well-understood, since $\cE_\alpha(\vtheta)$ is \emph{linear} w.r.t. $\alpha$ and $\cR(\alpha)$ in general is a (strong) concave function. As a result, $\alpha$ has a unique optimal. On the other hand, understanding the max player $\max_{\vtheta} \cE_\alpha(\vtheta)$ is important since it performs \emph{representation learning} in CL. It is also a hard problem because of non-convex optimization. 

We start with a specific case when $\vz$ is a deep linear network, i.e., $\vz = W(\vtheta)\vx$, where $W$ is the equivalent linear mapping for the deep linear network, and $\vtheta$ is the parameters to be optimized. Note that this covers many different kinds of deep linear networks, including VGG-like~\cite{saxe2013exact}, ResNet-like~\cite{hardt2016identity} and DenseNet-like~\cite{huang2017densely}. For notation brevity, we define $\con_\alpha[\vx] := \con_\alpha[\vx, \vx]$. 

\begin{restatable}[Representation learning in Deep Linear CL reparameterizes Principal Component Analysis (PCA)]{corollary}{anylinearispca}
\label{co:pca-reparam}
When $\vz = W(\vtheta)\vx$ with a constraint $W W^\top = I$, $\cE_\alpha$ is the objective of Principal Component Analysis (PCA) with reparameterization $W = W(\vtheta)$:  
\begin{equation}
\max_{\vtheta} \cE_\alpha(\vtheta) = \frac{1}{2}\tr(W(\vtheta)X_\alpha W^\top(\vtheta)) \quad\mathrm{s.t.\ } W W^\top = I
\label{eq:pca-reparameterization}
\end{equation}
here $X_\alpha := \con_\alpha[\vx]$ is the contrastive covariance of input $\vx$. 
\end{restatable}

As a comparison, in traditional Principal Component Analysis, the objective is~\citep{kokiopoulou2011trace}: $\frac{1}{2}\max_{W} \tr(W\cov_\mathrm{sample}[\vx]W^\top)$ subject to the constraint $WW^\top= I$, where $\cov_\mathrm{sample}[\vx]$ is the empirical covariance of the dataset (here it is one batch). Therefore, $X_\alpha$ can be regarded as a generalized covariance matrix, possibly containing negative eigenvalues. In the case of supervised CL (i.e,. pairs from the same/different labels are treated as positive/negative~\cite{khosla2020supervised}), then it is connected with Fisher's Linear Discriminant Analysis~\cite{fisher1936use}. 

Here we show a mathematically rigorous connection between CL and dimensional reduction, as suggested intuitively in~\cite{hadsell2006dimensionality}. Unlike traditional PCA, due to the presence of data augmentation, while symmetric, the contrastive covariance $X_\alpha$ is not necessarily a PSD matrix. Nevertheless, the intuition is the same: to find the direction that corresponds to maximal variation of the data. 

\def\linssl{\texttt{DeepLin}}

\begin{figure*}
    \centering
    \includegraphics[width=\textwidth]{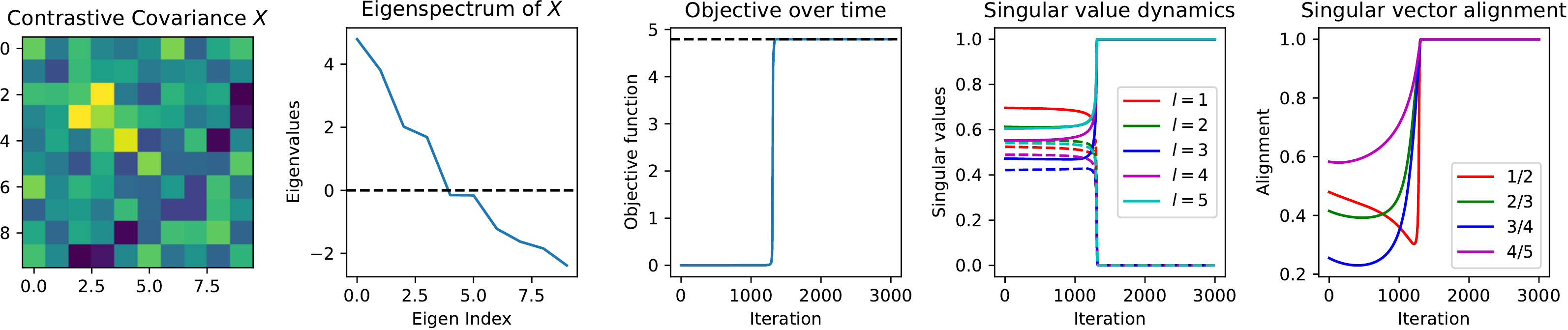}
    \caption{\small Dynamics of CL with multilayer ($L=5$) linear network (\linssl{}) with fixed $\alpha$. Running the training dynamics (Lemma~\ref{lemma:dyn-linear}) quickly leads to convergence towards the maximal eigenvalue of $X_\alpha$. For dynamics of singular value of $W_l$, the largest singular values (solid lines) converges to 1 while the second largest singular values (dashed lines) decay to $0$.}
    \label{fig:deep-lin-dynamics}
\end{figure*}

While it is interesting to discover that CL with deep linear network is essentially a reparameterization of PCA, it remains elusive that such a reparameterization leads to the same solution of PCA, in particular when the network is deep (and may contain local optima). Also, PCA has an overall end-to-end constraint $WW^\top = I$, while in network training, we instead use normalization layers and it is unclear whether they are equivalent or not.

In this section, we show for a specific deep linear model, almost all its local maxima of Eqn.~\ref{eq:pca-reparameterization} are global and it indeed solves PCA.  

\vspace{-0.05in}
\subsection{A concrete deep linear model}
\vspace{-0.05in}
We study a concrete deep linear network with parameters/weights $\vtheta := \{W_l\}_{l=1}^L$:
\begin{equation}
    \vz[i] := W_L W_{L-1} \ldots W_1\vx[i] \label{eq:model}
\end{equation}
Here $W_l \in \rr^{n_l\times n_{l-1}}$, $n_l$ is the number of nodes at layer $l$, $\vz[i]$ is the output of $\vx[i]$ and similarly $\vz[i']$ for $\vx[i']$. We use $\vtheta$ to represent the collection of weights at all layers. For convenience, we define the $l$-th layer activation $\vf_l[i] = W_l\vf_{l-1}[i]$. With this notation $\vf_0[i] = \vx[i]$ is the input and $\vz[i] = W_L\vf_{L-1}[i]$. 

We call this setting \linssl{}. The Jacobian matrix $W_{>l} := W_L W_{L-1} \ldots W_{l+1}$ and $W := W_{>0} = W_L W_{L-1} \ldots W_1$.
\begin{restatable}{lemma}{dynamicslinear}
\label{lemma:dyn-linear}
The training dynamics in \linssl{} is $\dot W_l = W_{>l}^\top W_{>l} W_l \con_\alpha[\vf_{l-1}]$ 
\end{restatable}
Note that $\con_\alpha[\vf_0] = \con_\alpha[\vx] = X_\alpha$. Similar to supervised learning~\cite{arora2018optimization,du2018algorithmic}, nearby layers are also balanced: $\frac{\dd}{\dd t} \left(W_lW_l^\top - W_{l+1}^\top W_{l+1}\right) = 0$.  

\def\onehidden{\texttt{ReLU2Layer1Hid}}
\def\noverlap{\texttt{ReLU2LayerNoOverlap}}
\def\relumodel{\texttt{ReLU2Layer}}

\vspace{-0.05in}
\subsection{Normalization Constraints}
\vspace{-0.05in}
\label{sec:normalization}
Note that if we just run the training dynamics (Lemma~\ref{lemma:dyn-linear}) without any constraints, $\|W_l\|_F$ will go to infinity. Fortunately, empirical works already suggest various ways of normalization to stabilize the network training. 

One popular technique in CL is $\ell_2$ normalization. It is often put right after the output of the network and before the loss function $\cL$~\cite{chen2020simclr,byol,He2020MomentumCF}, i.e., $\hat \vz[i] = \vz[i] / \|\vz[i]\|_2$. Besides, LayerNorm~\cite{ba2016layer} (i.e., $\hat \vf[i] = (\vf[i] - \mathrm{mean}(\vf[i])) / \mathrm{std}(\vf[i])$) is extensively used in Transformer-based models~\cite{xiong2020layer}. Here we show that for gradient flow dynamics of MLP models, such normalization layers conserve $\|W_l\|_F$ for any $l$ below it, regardless of loss function. 

\begin{restatable}{lemma}{lemmaltwolayernorm}
\label{lemma:l2-layer-norm}
For MLP, if the weight $W_l$ is below a $\ell_2$-norm or LayerNorm layer, then $\frac{\dd}{\dd t} \|W_l\|^2_F = 0$.
\end{restatable}
Note that Lemma~\ref{lemma:l2-layer-norm} also holds for nonlinear MLP with reversible activations, which includes ReLU (see SM). Therefore, without loss of generality, we consider the following complete objective for max player with \linssl{} (here $\Theta$ is the constraint set of the weights due to normalization): 
\begin{equation}
    \max_{\vtheta \in \Theta} \cE_\alpha(\vtheta) := \frac{1}{2}\tr(WX_\alpha W^\top), \quad \Theta := \{\vtheta: \ \ \|W_l\|_F = 1, 1\le l \le L\} \label{eq:pca-objective-main}
\end{equation}
\vspace{-0.05in}
\subsection{Representation Learning with \linssl{} is PCA}
\vspace{-0.05in}
\label{sec:local-is-global}
As one of our main contributions, the following theorem asserts that almost all local optimal solutions of Eqn.~\ref{eq:pca-objective-main} are global, and the optimal objective corresponds to the PCA objective. Note that~\cite{kawaguchi2016deep,laurent2018deep} proves no bad local optima for deep linear network in supervised learning, while here we give similar results for CL, and additionally we also give the (simple) rank-1 structure of all local optima.  

\begin{restatable}[Representation Learning with \linssl{} is PCA]{theorem}{mlpispca}
\label{thm:pca}
If $\lambda_{\max}(X_\alpha) > 0$, then for any local maximum $\vtheta\in \Theta$ of Eqn.~\ref{eq:pca-objective-main} whose $W^\top_{>1}W_{>1}$ has distinct maximal eigenvalue:
\begin{itemize}
    \item there exists a set of unit vectors $\{\vv_l\}_{l=0}^L$ so that $W_l = \vv_l \vv_{l-1}^\top$ for $1 \le l \le L$, in particular, $\vv_0$ is the unit eigenvector corresponding to $\lambda_{\max}(X_\alpha)$,
    \item $\vtheta$ is global optimal with objective $2\cE^* = \lambda_{\max}(X_\alpha)$. 
\end{itemize}
\end{restatable}
\begin{restatable}{corollary}{mlpispcabn}
\label{co:pre-filter-normalization}
If we additionally use per-filter normalization (i.e., $\|\vw_{lk}\|_2 = 1 / \sqrt{n_l}$), then Thm.~\ref{thm:pca} holds and $\vv_l$ is more constrained: $[\vv_l]_k = \pm 1 / \sqrt{n_l}$ for $1 \le l \le L-1$.
\end{restatable}

\textbf{Remark}. Here we prove that given fixed $\alpha$, maximizing $\cE_\alpha(\vtheta)$ gives rank-1 solutions for deep linear network. This conclusion is an extension of~\cite{jing2021understanding}, which shows weight collapsing happens if $\vtheta$ is 2-layer linear network and $\alpha$ is fixed. If the pairwise importance $\alpha$ is adversarial, then it may not lead to a rank-1 solution. In fact, $\alpha$ can magnify minimal eigen-directions and change the eigenstructure of $X_\alpha$ continuously. We leave it for future work. 

Note that the condition that ``$W^\top_{>1}W_{>1}$ has distinct maximal eigenvalue'' is important. Otherwise there are counterexamples. For example, consider 1-layer linear network $\vz = W_1\vx$, and $X_\alpha$ has duplicated maximal eigenvalues (with $\vu_1$ and $\vu_2$ being corresponding orthogonal eigenvectors), then $W_{>1}^\top W_{>1} = I$ (i.e., it has degenerated eigenvalues), and for any local maximal $W_1$, its row vector can be arbitrary linear combinations of $\vu_1$ and $\vu_2$ and thus $W_1$ is not rank-1. 

Compared to recent works~\cite{ji2021power} that also relates CL with PCA in linear representation setting using constant $\alpha$, our Theorem~\ref{thm:pca} has no statistical assumptions on data distribution and augmentation, and operates on vanilla InfoNCE loss and deep architectures.

\def\rank{\mathrm{rank}}
\def\aug{\mathrm{aug}}

\vspace{-0.05in}
\section{How Representation Learning Differs in Two-layer ReLU Network} 
\vspace{-0.05in}
\label{sec:relu-behavior}
So far we have shown that the max player $\max_{\vtheta}\cE_\alpha(\vtheta) := \frac{1}{2}\tr(\con_\alpha[\vz(\vtheta)])$ is essentially a PCA objective when the input-output mapping $\vz = W(\vtheta)\vx$ is linear. A natural question arises. What is the benefit of CL if its representation learning component has such a simple nature? Why can it learn a good representation in practice beyond PCA?

For this, nonlinearity is the key but understanding its role is highly nontrivial. For example, when the neural network model is nonlinear, Thm.~\ref{thm:general-E} and Corollary~\ref{co:minimax} holds but \emph{not} Corollary~\ref{co:pca-reparam}. Therefore, there is not even a well-defined $X_\alpha$ due to the fact that multiple hidden nodes can be switched on/off given different data input. Previous works~\cite{safran2018spurious,du2018gradient} also show that with nonlinearity, in supervised learning spurious local optima exist. 

Here we take a first step to analyze nonlinear cases. We study 2-layer models with ReLU activation $h(x) = \max(x, 0)$. We show that with a proper data assumption, the 2-layer model shares a \emph{modified} version of dynamics with its linear version, and the contrastive covariance term $X_\alpha$ (and its eigenstructure) remains well-defined and useful in nonlinear case. 

\vspace{-0.05in}
\subsection{The 2-layer ReLU network and data model}
\vspace{-0.05in}
We consider the bottom-layer weight $W_1 = [\vw_{11}, \vw_{12}, \ldots, \vw_{1K}]^\top$ with $\vw_{1k}$ being the $k$-th filter. For brevity, let $K=n_1$ be the number of hidden nodes. 
We still consider solution in the constraint set $\Theta$ (Eqn.~\ref{eq:pca-objective-main}), since Lemma~\ref{lemma:l2-layer-norm} still holds for ReLU networks. This model is named \relumodel{}. 

In addition, we assume the following data model: 
\begin{assumption}[Orthogonal mixture model within receptive field $R_k$]
\label{assumption:orth-mixture}
There exists a set of orthonormal bases $\{\bar\vx_{m}\}_{m=1}^M$ so that any input data $\vx[i] = \sum_m a_{m}[i] \bar\vx_{m}$ satisfies the property that $a_{m}[i]$ is \textbf{Nonnegative}: $a_{m}[i] \ge 0$, \textbf{One-hot}: for any $k$, $a_{m}[i] > 0$ for at most one $m$ and \textbf{Augmentation} only scales $\vx_k$ by a (sample-dependent) factor, i.e., $\vx[i'] = \gamma[i]\vx[i]$ with $\gamma[i] > 0$. 
\end{assumption}
Since all $\vx$ appears in the inner-product with the weight vectors $\vw_{1k}$, with a rotation of coordination, we can just set $\bar\vx_m = \ve_m$, where $\ve_m$ is the one-hot vector with $m$-th component being 1. In this case, $\vx\ge 0$ is always a one-hot vector with only at most only one positive entry. 

Intuitively, the model is motivated by sparsity: in each instantiation of $\vx$, there are very small number of activated modes and their linear combination becomes the input signal $\vx$. As we shall see, even with this simple model, the dynamics of ReLU network behaves very differently from the linear case.  

With this assumption, we only need to consider nonnegative low-layer weights and $X_\alpha$ is still a valid quantity for $\relumodel{}$: 
\begin{restatable}[Evaluation of \relumodel{}]{lemma}{lemmaevaluationtwolayer}
\label{lemma:negative-do-not-matter}
If Assumption~\ref{assumption:orth-mixture} holds, setting $\vw'_{1k} = \max(\vw_{1k}, 0)$ won't change the output of \relumodel{}. Furthermore, if $W_1 \ge 0$, then the formula for linear network $\cE_\alpha =\frac{1}{2}\tr(W_2W_1X_\alpha W_1^\top W_2^\top)$ still works for \relumodel{}. 
\end{restatable}

On the other hand, sharing the energy function $\cE_\alpha$ does not mean \relumodel{} is completely identical to its linear version. In fact, the dynamics follows its linear counterparts, but with important modifications:
\begin{restatable}[Dynamics of \relumodel{}]{theorem}{theoremdynamicstwolayer}
\label{thm:property-of-relu}
If Assumption~\ref{assumption:orth-mixture} holds, then the dynamics of \relumodel{} with $\vw_{1k} \ge 0$ is equivalent to linear dynamics with the \textbf{Sticky Weight rule}: any component that reaches 0 stays 0.
\end{restatable}
As we will see, this modification leads to very different dynamics and local optima in \relumodel{} from linear cases, even when there is only one ReLU node. 

\vspace{-0.05in}
\subsection{Dynamics in One ReLU node}
\vspace{-0.05in}
\label{sec:one-hidden-node}
Now we consider the dynamics of the simplest case: \relumodel{} with only 1 hidden node. In this case, $W^\top_{>1}W_{>1}$ is a scalar and thus $W_2^\top W_2 = \tr(W_2^\top W_2) = 1$. We only need to consider $\vw_1 \in \rr^{n_1}$, which is the only weight vector in the lower layer, under the constraint $\|W_1\|_F = \|\vw_1\|_2 = 1$ (Eqn.~\ref{eq:pca-objective-main}). We denote this setting as \onehidden{}. 

The dynamics now becomes very different from linear setting. Under linear network, according to Theorem~\ref{thm:pca}, $\vw_1$ converges to the largest eigenvector of $X_\alpha = \con_\alpha[\vx_1]$. For \onehidden{}, situation differs drastically: 
\begin{restatable}{theorem}{onenodedynamics}
\label{thm:relu-1-node-dynamics}
If Assumption~\ref{assumption:orth-mixture} holds, then in \onehidden{}, $\vw_1\rightarrow \ve_m$ for certain $m$. 
\end{restatable}
Intuitively, this theorem is achieved by closely tracing the dynamics. When the number of positive entries of $\vw_1$ is more than 1, the linear dynamics always hits the boundary of the polytope $\vw_1 \ge 0$, making one of its entry be zero, and stick to zero due to sticky weight rule. This procedure repeats until there is only one survival positive entry in $\vw_1$. 

Overall, this simple case already shows that nonlinear landscape can lead to many local optima: for any $m$, $\vw_1 = \ve_m$ is one local optimal. Which one the training falls into depends on weight initialization, and critically affects the properties of per-trained models. 

\def\rank{\mathrm{rank}}
\def\aug{\mathrm{aug}}

\begin{figure}
    \centering
    \includegraphics[width=.9\textwidth]{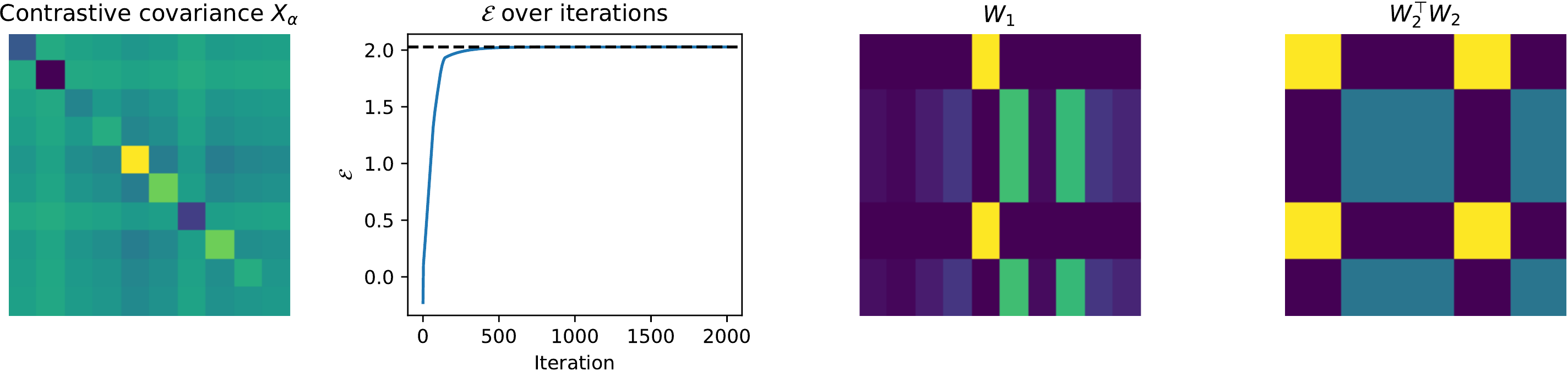}
    \vspace{-0.1in}
    \caption{\small Theorem~\ref{thm:relu-diversity} shows that training \relumodel{} could lead to more diverse hidden weight patterns beyond rank-1 solution obtained in the linear case (shown in right two figures: converged $W_1$ and $W_2^\top W_2$).}
    \label{fig:relu-diversity}
\end{figure}

\vspace{-0.05in}
\subsection{Multiple hidden nodes}
\vspace{-0.05in}
\label{sec:multiple-hidden-one-receptive-field}
For complicated situations like multiple hidden units, completely characterizing the training dynamics like Theorem~\ref{thm:relu-1-node-dynamics} becomes hard (if not impossible). Instead, we focus on fixed point analysis.

For deep linear model, using multiple hidden units does not lead to any better solutions. According to Thm.~\ref{thm:pca}, at local optimal, $W_1 = \vv_1\vv^\top_0$. This means that the weights $\vw_{1k}$, which are row vectors of $W_1$, are just a scaled version of the maximal eigenvector $\vv_0$ of $X_\alpha$. 
Moreover, this is independent of the eigenstructure of $X_\alpha$ as long as $\lambda_{\max}(X_\alpha) > 0$. 

In \relumodel{}, the situation is a bit different. Thm.~\ref{thm:relu-diversity} shows that these hidden nodes are (slightly) more diverse. Fig.~\ref{fig:relu-diversity} shows one such example. The intuition here is that in nonlinear case, rank-1 structure of the critical points may be replaced with low-rank structures.  

\begin{restatable}[\relumodel{} encourages diversity]{theorem}{reludiversity}
\label{thm:relu-diversity}
If Assumption~\ref{assumption:orth-mixture} holds, then for any local optimal $(W_2, W_1) \in \Theta$ of $\relumodel{}$ with $\cE > 0$, either $W_1 = \vv\ve_m^\top$ for some $m$ and $\vv \ge 0$, or $\rank(W_1) > 1$.  
\end{restatable}

\vspace{-0.1in}
\section{Experiments}
\vspace{-0.05in}
\label{sec:experiment}
\label{sec:different-loss-formulation}
We evaluate our \method{} framework (Def.~\ref{def:method}) in CIFAR10~\cite{krizhevsky2009learning} and STL-10~\cite{coates2011analysis} with ResNet18~\cite{he2016deep}, and compare the downstream performance of multiple losses, with regularizers taking the form of $\cR(\alpha) = \sum_i\sum_{j\neq i} r(\alpha_{ij})$ with a constraint $\sum_{j\neq i} \alpha_{ij} = 1$. Here $r$ can be different concave functions: 
\begin{itemize}
    \item (\method{}-$r_H$) Entropy regularizer $r_H(\alpha_{ij}) = -\tau \alpha_{ij}\log \alpha_{ij}$;
    \item (\method{}-$r_\gamma$) Inverse regularizers $r_\gamma(\alpha_{ij}) = \frac{\tau}{1-\gamma} \alpha_{ij}^{1-\gamma}$ ($\gamma > 1$). 
    \item (\method{}-$r_s$) Square regularizer $r_s(\alpha_{ij}) = -\frac{\tau}{2}\alpha^2_{ij}$.
\end{itemize}
Besides, we also compare with the following:
\begin{itemize}
\item Minimizing InfoNCE or quadratic loss: $\min_{\vtheta}\cL(\vtheta)$ for $\cL \in \{\cL_{nce}, \cL_{quadratic}\}$.
\item Setting $\alpha$ as InfoNCE (Eqn.~\ref{eq:infonce-alpha-beta}) and backpropagates through $\alpha=\alpha(\vtheta)$ with respect to $\vtheta$. 
\item (\method{}-direct) Directly setting $\alpha$ (here $p > 1$):
\begin{equation}
     \alpha_{ij} = \frac{\exp(-d^p_{ij}/\tau)}{\sum_j\exp(-d^p_{ij}/\tau)}   \label{eq:direct-alpha} 
\end{equation}
\end{itemize}
For inverse regularizer $r_\gamma$, we pick $\gamma=2$ and $\tau=0.5$; for direct-set $\alpha$, we pick $p = 4$ and $\tau=0.5$; for square regularizer, we use $\tau=5$. All training is performed with Adam~\cite{kingma2014adam} optimizer. Code is written in PyTorch and a single modern GPU suffices for the experiments.

\begin{table}[t]
\centering
\small
\caption{\small Comparison over multiple loss formulations (ResNet18 backbone, batchsize 128). Top-1 accuracy with linear evaluation protocol. Temperature $\tau=0.5$ and learning rate is $0.01$. \textbf{Bold} is highest performance and \textcolor{blue}{blue} is second highest. Each setting is repeated 5 times with different random seeds. \label{tbl:diff-formulation}}
\setlength{\tabcolsep}{1pt}
\begin{tabular}{l||c|c|c||c|c|c}
&\multicolumn{3}{c||}{\emph{CIFAR-10}} & \multicolumn{3}{c}{\emph{STL-10}} \\
\hline 
         & 100 epochs & 300 epochs & 500 epochs & 100 epochs & 300 epochs & 500 epochs \\ 
\hline
$\cL_{quadratic}$ & $63.59\pm 2.53$ & $73.02\pm 0.80$ & $73.58\pm 0.82$ & $55.59\pm 4.00$ & $64.97\pm 1.45$ & $67.28\pm 1.21$ \\
$\cL_{nce}$ & $84.06\pm 0.30$ & $87.63\pm 0.13$ & $87.86\pm 0.12$ & $78.46\pm 0.24$ & $82.49\pm 0.26$ & $83.70\pm 0.12$ \\
backprop $\alpha(\vtheta)$ & $83.42\pm 0.25$ & $87.18\pm 0.19$ & $87.48\pm 0.21$ & $77.88\pm 0.17$ & $81.86\pm 0.30$ & $83.19\pm 0.16$  \\
\method{}-$r_H$ & $84.27\pm 0.24$ & $\textcolor{blue}{87.75\pm 0.25}$ & $\textcolor{blue}{87.92\pm 0.24}$ & $\textcolor{blue}{78.53\pm 0.35}$ & $\textcolor{blue}{82.62\pm 0.15}$ & $\textcolor{blue}{83.74\pm 0.18}$ \\
\method{}-$r_\gamma$ & $83.72\pm 0.19$ &  $87.51\pm 0.11$ & $87.69\pm 0.09$ & $78.22\pm 0.28$ & $82.19\pm 0.52$ & $83.47\pm 0.34$  \\ 
\method{}-$r_s$ & $\textcolor{blue}{84.72\pm 0.10}$ & $86.62\pm 0.17$ & $86.74\pm 0.15$ & $76.95\pm 1.06$ & $80.64\pm 0.77$ &  $81.65\pm 0.59$ \\
\method{}-direct & $\mathbf{85.11\pm 0.19}$ & $\mathbf{87.93\pm 0.16}$ &  $\mathbf{88.09\pm 0.13}$ & $\mathbf{79.32\pm 0.36}$ & $\mathbf{82.95\pm 0.17}$ & $\mathbf{84.05\pm 0.20}$
\end{tabular}
\vspace{-0.1in}
\end{table}

\begin{table}[]
\small
\setlength{\tabcolsep}{1pt}
    \begin{tabular}{l||c|c|c||c|c|c|}
&\multicolumn{3}{c||}{\emph{ResNet18 Backbone}} & \multicolumn{3}{c|}{\emph{ResNet50 Backbone}} \\
\hline
&\multicolumn{6}{c|}{\emph{CIFAR-100}} \\
\hline 
         & 100 epochs & 300 epochs & 500 epochs & 100 epochs & 300 epochs & 500 epochs \\ 
\hline
$\cL_{nce}$ & $55.70 \pm 0.37$ &	$59.71 \pm 0.36$ & $59.89 \pm 0.34$ & $60.16 \pm 0.48$ &	$65.40 \pm 0.31$ & $65.53 \pm 0.30$ \\ 
\method{}-direct & $\mathbf{57.63\pm 0.07}$ & $\mathbf{60.12\pm 0.26}$ & $\mathbf{60.27\pm 0.29}$ & $\mathbf{62.93\pm 0.28}$ & $\mathbf{65.84\pm 0.14}$ & $\mathbf{65.87\pm 0.21}$ \\
\hline
&\multicolumn{6}{c|}{\emph{CIFAR-10}} \\
\hline
$\cL_{nce}$ &  $84.06\pm 0.30$ & $87.63\pm 0.13$ & $87.86\pm 0.12$ & $86.39 \pm 0.16$ &	$89.97 \pm 0.14$ & $90.19 \pm 0.23$ \\ 
\method{}-direct & $\mathbf{85.11\pm 0.19}$ & $\mathbf{87.93\pm 0.16}$ &  $\mathbf{88.09\pm 0.13}$ & $\mathbf{87.79\pm 0.25}$ & $\mathbf{90.41\pm 0.18}$ & $\mathbf{90.50\pm 0.21}$ \\ 
\hline
&\multicolumn{6}{c|}{\emph{STL-10}} \\
\hline
$\cL_{nce}$ & $78.46\pm 0.24$ & $82.49\pm 0.26$ & $83.70\pm 0.12$ & $81.64 \pm 0.24$ & $86.57 \pm 0.17$ & $\mathbf{87.90 \pm 0.22}$ \\ 
\method{}-direct & $\mathbf{79.32\pm 0.36}$ & $\mathbf{82.95\pm 0.17}$ & $\mathbf{84.05\pm 0.20}$ & $\mathbf{83.20\pm 0.25}$ & $\mathbf{87.17\pm 0.14}$ & $87.85\pm 0.21$ \\ 
    \end{tabular}
    \vspace{-0.1in}
    \caption{\small More experiments with ResNet18/ResNet50 backbone on CIFAR-10, STL-10 and CIFAR-100. Batchsize is 128. For ResNet18, learning rate is 0.01; for ResNet50, learning rate is 0.001.}
    \label{tab:more-ablation}
    %
\end{table}

The results are shown in Tbl.~\ref{tbl:diff-formulation}. We can see that (1) backpropagating through $\alpha(\vtheta)$ is  worse, justifying our perspective of coordinate-wise optimization, (2) our proposed \method{} works for different regularizers, (3) using different regularizer leads to comparable or better performance than original InfoNCE $\cL_{nce}$, (4) the pairwise importance $\alpha$ does not even need to come from a minimization process. Instead, we can directly set $\alpha$ based on pairwise squared distances $d^2_{ij}$ and $d^2_i$. For \method{}-direct, the performance is slightly worse if we do not normalize $\alpha_{ij}$ (i.e., $\alpha_{ij} := \exp(-d_{ij}^p/\tau)$). It seems that for strong performance, $\frac{\dd r}{\dd \alpha_{ij}}$ should go to $+\infty$ when $\alpha_{ij} \rightarrow 0$. Regularizers that do not satisfy this condition (e.g., squared regularizer $r_s$) may not work as well. 

Tbl.~\ref{tab:more-ablation} shows more experiments with different backbones (e.g., ResNet50) and more complicated datasets (e.g., CIFAR-100). Overall, we see consistent gains of $\alpha$-CL over InfoNCE in early stages of the training (e.g., 1-2 point of absolute percentage gain) and comparable performance at 500 epoch. More ablations on batchsizes and exponent $p$ in Eqn.~\ref{eq:direct-alpha} are provided in Appendix~\ref{sec:appendix-more-exp}. 

\vspace{-0.07in}
\section{Conclusion and Future Work}
\vspace{-0.07in}
We provide a novel perspective of contrastive learning (CL) via the lens of coordinate-wise optimization and propose a unified framework called \method{} that not only covers a broad family of loss functions including InfoNCE, but also allows a direct set of importance of sample pairs. Preliminary experiments on CIFAR10/STL-10/CIFAR100 show comparable/better performance with the new loss than InfoNCE. Furthermore, we prove that with deep linear networks, the representation learning part is equivalent to Principal Component Analysis (PCA). In addition, we also extend our analysis to representation learning in 2-layer ReLU network, shedding light on the important difference in representation learning for linear/nonlinear cases. 

\textbf{Future work}. Our framework \method{} turns various loss functions into a unified framework with different choices of pairwise importance $\alpha$ and how to find good choices remains open. Also, we mainly focus on representation learning with fixed pairwise importance $\alpha$. However, in the actual training, $\alpha$ and $\vtheta$ change concurrently. Understanding their interactions is an important next step. Finally, removing Assumption~\ref{assumption:orth-mixture} in ReLU analysis is also an open problem to be addressed later.

\bibliography{references}
\bibliographystyle{icml2022}

\clearpage

\ifarxiv
\else
\section*{Checklist}

\begin{enumerate}

\item For all authors...
\begin{enumerate}
  \item Do the main claims made in the abstract and introduction accurately reflect the paper's contributions and scope?
    \answerYes{}
  \item Did you describe the limitations of your work?
    \answerYes{}
  \item Did you discuss any potential negative societal impacts of your work?
    \answerNA{}
  \item Have you read the ethics review guidelines and ensured that your paper conforms to them?
    \answerYes{}
\end{enumerate}

\item If you are including theoretical results...
\begin{enumerate}
  \item Did you state the full set of assumptions of all theoretical results?
    \answerYes{}
        \item Did you include complete proofs of all theoretical results?
    \answerYes{}
\end{enumerate}

\item If you ran experiments...
\begin{enumerate}
  \item Did you include the code, data, and instructions needed to reproduce the main experimental results (either in the supplemental material or as a URL)?
    \answerYes{}
  \item Did you specify all the training details (e.g., data splits, hyperparameters, how they were chosen)?
    \answerYes{}
        \item Did you report error bars (e.g., with respect to the random seed after running experiments multiple times)?
    \answerYes{}
        \item Did you include the total amount of compute and the type of resources used (e.g., type of GPUs, internal cluster, or cloud provider)?
    \answerYes{}
\end{enumerate}

\item If you are using existing assets (e.g., code, data, models) or curating/releasing new assets...
\begin{enumerate}
  \item If your work uses existing assets, did you cite the creators?
    \answerYes{}
  \item Did you mention the license of the assets?
    \answerNA{}
  \item Did you include any new assets either in the supplemental material or as a URL?
    \answerNo{}
  \item Did you discuss whether and how consent was obtained from people whose data you're using/curating?
    \answerNA{}
  \item Did you discuss whether the data you are using/curating contains personally identifiable information or offensive content?
    \answerNA{}
\end{enumerate}

\item If you used crowdsourcing or conducted research with human subjects...
\begin{enumerate}
  \item Did you include the full text of instructions given to participants and screenshots, if applicable?
    \answerNA{}
  \item Did you describe any potential participant risks, with links to Institutional Review Board (IRB) approvals, if applicable?
    \answerNA{}
  \item Did you include the estimated hourly wage paid to participants and the total amount spent on participant compensation?
    \answerNA{}
\end{enumerate}

\end{enumerate}
\fi


\onecolumn
\appendix

\section{Proofs}
\subsection{Section~\ref{sec:general-game-framework}}

\gradientmatch*
\begin{proof}
By the definition of gradient descent, we have for any component $\theta$ in a high-dimensional vector $\vtheta$: 
\begin{equation}
    -\frac{\partial \cL}{\partial \theta} = -\sum_{i=1}^N \frac{\partial \vz[i]}{\partial \theta}\frac{\partial \cL}{\partial \vz[i]} + \frac{\partial \vz[i']}{\partial \theta}\frac{\partial \cL}{\partial \vz[i']} \label{eq:dot-wl-original}
\end{equation}
Here we use the ``Denominator-layout notation'' and treat $\frac{\partial \cL}{\partial \vz[i]}$ as a column vector while $\frac{\partial \vz[i]}{\partial \theta}$ as a row vector. Using Lemma~\ref{lemma:gradient-form}, we have:
\begin{equation}
     -\frac{\partial \cL}{\partial \theta} = \con_\alpha\left[\frac{\partial \vz}{\partial \theta}, \vz^\top\right]
\end{equation}
On the other hand, treating $\alpha$ as independent variables of $\vtheta$, we compute (here $o_k$ is the $k$-th component of $\vz$):
\begin{equation}
    \frac{\partial \cE_\alpha}{\partial \theta} = \frac{1}{2} \sum_k \con_\alpha\left[\frac{\partial o_k}{\partial \theta}, o_k\right] + \frac{1}{2}\sum_k\con_\alpha\left[o_k, \frac{\partial o_k}{\partial \theta}\right]
\end{equation}
For scalar $x$ and $y$, $\con_\alpha[x,y] =\con_\alpha[y,x]$ and $\sum_k \con_\alpha[a_k, b_k] = \con_\alpha[\va, \vb^\top]$ for row vector $\va$ and column vector $\vb$. Therefore, 
\begin{equation}
    \frac{\partial \cE_\alpha}{\partial \vtheta} = \con_\alpha\left[\frac{\partial \vz}{\partial \vtheta}, \vz^\top\right]
\end{equation}

Therefore, we have 
\begin{equation}
    \frac{\partial \cE_\alpha}{\partial \vtheta} = -\frac{\partial \cL}{\partial \vtheta}
\end{equation} 
and the proof is complete.  
\end{proof}

\infoncealpha*
\begin{proof}
We just need to solve the internal minimizer w.r.t. $\alpha$. Note that each $\alpha_{i}$ can be optimized independently. 

First, we know that $\cE_\alpha(\vtheta) := \frac{1}{2}\tr \con_\alpha [\vz,\vz]$ can be written as:
\begin{eqnarray}
    \cE_\alpha(\vtheta) &=& \frac{1}{2}\sum_{i\neq j} \alpha_{ij} \left[\tr (\vz[i] - \vz[j]) (\vz[i] - \vz[j])^\top - \tr(\vz[i] - \vz[i'])(\vz[i] - \vz[i'])^\top\right] \\
    &=& \frac{1}{2}\sum_{i\neq j} \alpha_{ij} \left[\|\vz[i] - \vz[j]\|_2^2 - \|\vz[i] - \vz[i']\|^2_2\right] \\
    &=& \sum_{i\neq j} \alpha_{ij} \left(d^2_{ij} - d^2_i\right) 
\end{eqnarray}

For each $\alpha_{i\cdot}$, applying Lemma~\ref{lemma:linear-entropy} with $c_{ij} = d^2_{ij} - d^2_i$, the optimal solution $\alpha$ is:
\begin{eqnarray}
\alpha_{ij} &=& \frac{1}{\tau} \exp\left(-\frac{c_{ij}}{\tau}\right) \phi'\left(\sum_{j\neq i} \exp\left(-\frac{c_{ij}}{\tau}\right)\right) \\
&=& \frac{1}{\tau} \exp\left(\frac{d_i^2 - d^2_{ij}}{\tau}\right) \phi'\left(\sum_{j\neq i} \exp\left(\frac{d^2_i - d^2_{ij}}{\tau}\right)\right) \\
&=& \psi'(d^2_i - d^2_{ij}) \phi'\left(\sum_{j\neq i} \psi(d^2_i-d^2_{ij})\right)  \\
&=& \psi'(d^2_i - d^2_{ij}) \phi'(\xi_i)
\end{eqnarray}
which coincides with Eqn.~\ref{eq:alpha} that is from the gradient descent rule of the loss function $\cL_{\phi, \psi}$. 

In particular, for InfoNCE, we have $\phi(x) = \tau\log(\epsilon + x)$, $\phi'(x) = \tau / (x + \epsilon)$ and therefore:
\begin{equation}
\alpha_{ij} = \frac{\exp( (d^2_i-d^2_{ij}) / \tau)}{\epsilon + \sum_{j\neq i} \exp( (d^2_i - d^2_{ij}) / \tau)} = \frac{\exp(-d^2_{ij} / \tau)}{\epsilon\exp(-d^2_i/\tau) + \sum_{j\neq i} \exp( - d^2_{ij} / \tau)} 
\end{equation}
which is exactly the coefficients $\alpha_{ij}$ directly computed during minimization of $\cL_{nce}$. If $\epsilon = 0$, then the constraint becomes $\sum_{j\neq i} \alpha_{ij} = 1$ and we have:
\begin{equation}
     \alpha_{ij} = \frac{\exp(-d^2_{ij} / \tau)}{\sum_{j\neq i} \exp( - d^2_{ij} / \tau)}  
\end{equation}
That is, the coefficients $\alpha$ does not depend on intra-augmentation squared distance $d^2_i$.
\end{proof}

\minimaxeq*
\begin{proof}
The proof naturally follows from the conclusion of Theorem~\ref{thm:general-E} and Theorem~\ref{thm:infonce-alpha}. 
\end{proof}

\subsection{Section~\ref{sec:linear-is-pca}}
\anylinearispca*
\begin{proof}
Notice that in deep linear setting, $\vz = W(\vtheta)\vx$ where $W(\vtheta)$ does not dependent on specific samples. Therefore, $\con_\alpha[\vz, \vz] = W(\vtheta)\con_\alpha[\vx, \vx]W^\top(\vtheta) = W(\vtheta)X_\alpha W^\top(\vtheta) $. 
\end{proof}

\dynamicslinear*
\begin{proof}
We can start from Eqn.~\ref{eq:dot-wl-original} directly and takes out $J^\top_{>l}$. This leads to 
\begin{equation}
    \dot W_l = J^\top_{>l} \left(\sum_{i=1}^N \frac{\partial \cL}{\partial \vz[i]} \vf^\top_{l-1}[i] + \frac{\partial \cL}{\partial \vz[i']} \vf^\top_{l-1}[i']\right) = J^\top_{>l} \con_\alpha[\vz, \vf_{l-1}]  \label{eq:dot-wl-linear}
\end{equation}
Using that $\vz = J_{\ge l} \vf_{l-1}$ leads to the conclusion. If the network is linear, then $J_{>l}^\top[i] = J^\top_{>l}$ is a constant. Then we can take the common factor $J^\top_{>l}J_{\ge l}$ out of the summation, yield $\dot W_l = J^\top_{>l}J_{\ge l} F_{l-1}$. Here $F_{l} := \con_\alpha[\vf_l]$ is the contrastive covariance at layer $l$.  
\end{proof}

\subsubsection{Section~\ref{sec:normalization}}
\label{sec:appendix-normalization}
\begin{figure}
    \centering
    \includegraphics[width=0.5\textwidth]{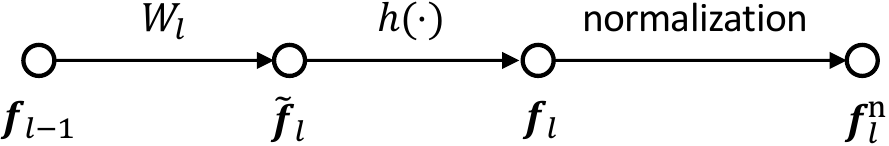}
    \caption{Notations on normalization (Sec.~\ref{sec:appendix-normalization}).}
    \label{fig:normalization}
\end{figure}

For this we talk about more general cases where the deep network is nonlinear. Let $h(\cdot)$ be the point-wise activation function and the network architecture looks like the following:
\begin{equation}
    \vz[i] := W_L h(W_{L-1}(h(\ldots W_1\vx[i]))) \label{eq:model-appendix}
\end{equation}
We consider the case where $h(\cdot)$ satisfies the following constraints:
\begin{definition}[Reversibility~\cite{tian2020understanding} / Homogeneity~\cite{du2018algorithmic}]
The activation function $h(x)$ satisfies $h(x) = h'(x) x$.
\end{definition}
This is satisfied by linear, ReLU, leaky ReLU and many polynomial activations (with an additional constant). With this condition, we have $\vf_l[i] = D_l W_l\vf_{l-1}[i]$, where $D_l = D_l(\vx[i]) := \diag[h'(\vw_{lk}^\top \vf_{l-1}[i])] \in \rr^{n_l\times n_l}$ is a diagonal matrix. For ReLU activation, the diagonal entry of $D_l$ is binary. 

\def\iin{\mathrm{in}}
\def\out{\mathrm{out}}

\begin{definition}[Reversible Layers~\cite{tian2020understanding}]
A layer is reversible if there exists $J[i]$ so that $\vf_\out[i] = J[i] \vf_\iin[i]$ and $\vg_\iin[i] = J^\top[i] \vg_\out[i]$ for each sample $i$. 
\end{definition}
It is clear that linear layers, ReLU and leaky ReLU are reversible. Lemma~\ref{lemma:normalization-lemma} tells us that $\ell_2$-normalization and LayerNorm are also reversible. 

\lemmaltwolayernorm*
\begin{proof}
See Lemma~\ref{lemma:layernorml2} that proves more general cases. 
\end{proof}

\subsubsection{Section~\ref{sec:local-is-global}}
\begin{definition}[Aligned-rank-1 solution]
\label{def:rank-1}
A solution $\vtheta = \{W_l\}_{l=1}^L$ is called aligned-rank-1, if there exists a set of unit vectors $\{\vv_l\}_{l=0}^L$ so that $W_l = \vv_l \vv_{l-1}^\top$ for $1 \le l \le L$. 
\end{definition}

\mlpispca*
\begin{proof}
A necessary condition for $\vtheta$ to be the local maximum is the critical point condition (here $\lambda_{l-1}$ is some constant):
\begin{equation}
    W^\top_{>l} W_{>l} W_l F_{l-1} = \lambda_{l-1} W_l \label{eq:critical-point}
\end{equation}
Right multiplying $W_l$ on both sides of the critical point condition for $W_l$, and taking matrix trace, we have:
\begin{equation}
    2\cE(\vtheta) = \tr(W^\top_{>l} W_{>l} W_l F_{l-1} W_l^\top) = \tr(\lambda_{l-1} W_l W_l^\top) = \lambda_{l-1} 
\end{equation}
Therefore, all $\lambda_l$ are the same, denoted as $\lambda$, and they are equal to the objective value. 

Now let's consider $l = 1$. Then we have:
\begin{equation}
    W^\top_{>1} W_{>1} W_1 X = \lambda W_1
\end{equation}
Applying $\vec(AXB) = (B^\top \otimes A)\vec(X)$, we have:
\begin{equation}
     (X \otimes W^\top_{>1} W_{>1}) \vec(W_1) =  \lambda\vec(W_1) \label{eq:eigfunc-kron-form}
\end{equation}
with the constraint that $\|\vec(W_1)\|^2_2 = \|W_1\|_F^2 = 1$. Similarly, we have $2\cE(\vtheta) = \lambda$. 

We then prove that $\lambda$ is the largest eigenvalue of $X \otimes W^\top_{>1} W_{>1}$. We prove by contradiction. If not, then $\vec(W_1)$ is not the largest eigenvector, then there is always a direction $W_1$ can move, while respecting the constraint $\|W_1\|_F=1$ and keeping $W_{>1}$ fixed, to make $\cE(\vtheta)$ strictly larger. Therefore, for any local maximum $\vtheta$, $\lambda$ has to be the largest eigenvalue of $X \otimes W^\top_{>1} W_{>1}$.  

Let $\{\vv_{0m}\}$ be the orthonormal basis of the eigenspace of $\lambda_{\max}(X)$ and $\vu$ be the (unique by the assumption) maximal unit eigenvector of $W^\top_{>1} W_{>1}$. Then $\vec(W_1) = \sum_m c_m \vv_{0m} \otimes \vu$ where $\sum_m c_m^2 = 1$, or $\vec(W_1) = \vv_0 \otimes \vu$ where the unit vector $\vv_0 := \sum_m c_m \vv_{0m}$. Plug $\vec(W_1) = \vv_0 \otimes \vu$ into Eqn.~\ref{eq:eigfunc-kron-form}, notice that $\vv_0$ is still the largest eigenvector of $X$, and we have $\lambda = \lambda_{\max}(X) \|W_{>1}\vu\|^2_2$.  

Now we show that $\lambda_{\max}(W^\top_{>1} W_{>1}) = \|W_{>1}\|^2_2 = 1$. If not, i.e., $\|W_{>1}\|_2 < 1$, then first by Lemma~\ref{lemma:norm2rank1}, we know that $\cW_{>1}:= \{W_L, W_{L-1}, \ldots, W_2\}$ must not be aligned-rank-1. Since $W_{>1}^\top W_{>1}$ is PSD and has unique maximal eigenvector $\vu$, the eigenvalue associated with $\vu$ must be strictly positive and thus $W_{>1}\vu \neq 0$. 

Then by Lemma~\ref{lemma:mm-property}, $\cW_{>1}$ is not a local maximum of $\cJ(\cW_{>1};\vu) := \max_{\cW_{>1}} \|W_{>1}\vu\|_2$ s.t. $\|W_l\|_F=1$, which means that there exists $\cW'_{>1}:= \{W'_L, W'_{L-1},\ldots, W'_2\}$ in the local neighborhood of $\cW_{>1}$ so that
\begin{itemize}
    \item $\|W'_l\|_F=1$ for $2\le l\le L$. That is, $W'$ is a feasible solution of $\cJ$.
    \item $\cJ(\cW'_{>1}) := \|W'_{>1}\vu\|_2 > \|W_{>1}\vu\|_2 = \cJ(\cW_{>1})$. 
\end{itemize}
Then let $\vtheta' := \{W'_L, W'_{L-1}, \ldots, W'_2, W_1\}$ which is a feasible solution to \linssl{}, we have:
\begin{eqnarray}
    2\cE(\vtheta') &=& \vec^\top(W_1) (X \otimes {W'}^\top_{>1} W'_{>1}) \vec(W_1) \\
    &=& (\vv^\top_{0}\otimes \vu^\top) (X \otimes {W'}^\top_{>1} W'_{>1}) (\vv_{0}\otimes \vu) \\
    &=& \lambda_{\max}(X) \|W'_{>1}\vu\|^2_2 \\
    &>& \lambda_{\max}(X) \|W_{>1}\vu\|^2_2 = \lambda = 2\cE(\vtheta)
\end{eqnarray}
This means that $\vtheta$ is not a local maximum, which is a contradiction. Note that $\vtheta'$ is not necessarily a critical point (and Eqn.~\ref{eq:critical-point} may not hold for $\vtheta'$). 

Therefore, $\lambda_{\max}(W^\top_{>1} W_{>1}) = \|W_{>1}\|^2_2 = 1$ and thus $2\cE(\vtheta) = \lambda = \lambda_{\max}(X)$.  

Since $\|W_{>1}\|_2 = 1$, again by Lemma~\ref{lemma:norm2rank1}, $W_{L:2}$ is aligned-rank-1 and $W_{>1} = \vv_L \vv_1^\top$ is also a rank-1 matrix. $W_{>1}^\top W_{>1} = \vv_1\vv_1^\top$ has a unique maximal eigenvector $\vv_1$.  Therefore $\vec(W_1) = \vv_0\otimes \vv_1$, or $W_1 = \vv_1\vv_0^\top$. As a result, $\vtheta := \{W_{L :2}, W_1\}$ is aligned-rank-1. 

Finally, since all local maxima have the same objective value $2\cE=\lambda_{\max}(X)$, they are all global maxima. 
\end{proof}

\textbf{Remarks}. Leveraging similar proof techniques, we can also show that with BatchNorm layers, the local maxima are more constrained. From Lemma~\ref{lemma:bn-weight} we knows that if each hidden node is covered with BatchNorm, then its fan-in weights are conserved. Therefore, without loss of generality, we could set the per-filter normalization: $\|\vw_{lk}\|_2 = 1$. In this case we have:

\begin{definition}[Aligned-uniform solution]
A solution $\vtheta$ is called aligned-uniform, if it is aligned-rank-1, and $[\vv_l]_k = \pm 1 / \sqrt{n_l}$ for $1 \le l \le L-1$. The two end-point unit vectors ($\vv_0$ and $\vv_L$) can still be arbitrary.  
\end{definition}

\mlpispcabn*
\begin{proof}
Leveraging Lemma~\ref{lemma:bn-reg} in Theorem~\ref{thm:pca} yields the conclusion. 
\end{proof}

\textbf{Remark}. We could see that with BatchNorm, the optimization problem is more constrained, and the set of local maxima have less degree of freedom. This makes optimization better behaved.   

\def\linear{\mathrm{linear}}
\def\relu{\mathrm{relu}}

\subsection{Section~\ref{sec:relu-behavior}}
\lemmaevaluationtwolayer*
\begin{proof}
For the first part, we just want to prove that if Assumption~\ref{assumption:orth-mixture} holds, then a 2-layer ReLU network with weights $\vw_{1k}$ and $W_2$ has the same activation as another ReLU network with  $\vw'_{1k} = \max(\vw_{1k},0) \ge 0$ and $W'_2=W_2$.

We are comparing the two activations:
\begin{eqnarray}
    f_{1k} &=& \max\left(\sum_m w_{1km} x_{km}, 0\right) \\
    f'_{1k} &=& \max\left(\sum_m \max(w_{1km}, 0) x_{km}, 0\right) = \sum_m \max(w_{1km}, 0) x_{km}
\end{eqnarray}
The equality is due to the fact that $\vx_k \ge 0$ (by nonnegativeness). Now we consider two cases. 

\textbf{Case 1}. If all $\vw_{1k} \ge 0$ then obviously they are identical.

\textbf{Case 2}. If there exists $m$ so that $w_{1km} < 0$. The only situation that the difference could happen is for some specific $\vx_k[i]$ so that $x_{km}[i] > 0$. By Assumption~\ref{assumption:orth-mixture}(one-hotness), for $m'\neq m$, $x_{km}[i] = 0$ so the gate $d_k[i] = \mathbb{I}(\vw_{1k}^\top \vx_k > 0) = 0$. On the other hand, ${\vw'}_{1k}^\top \vx_k = 0$ so $d_k'[i] = 0$. 

Therefore, in all situations, $f_{1k} = f'_{1k}$. 

For the second part, since $W_1 \ge 0$ and all input $\vx \ge 0$ by non-negativeness, all gates are open and the energy $\cE_\alpha$ of \relumodel{} is the same as the linear model. 
\end{proof}

\theoremdynamicstwolayer*
\begin{proof}
Let $\vw_{1k} \ge 0$ be the $k$-th filter to be considered and $w_{1km} \ge 0$ its $m$-th component. Consider a linear network with the same weights ($\vw'_{1k} = \vw_{1k}$ and $W_2' = W_2$) with only the ReLU activation removed. 

Now we consider the gradient rule of the ReLU network and the corresponding linear network with a sticky weight rule (here $g_k[i]$ is the backpropagated gradient sent to node $k$ for sample $i$, and $d_k[i]$ is the binary gating for sample $i$ at node $k$): 
\begin{eqnarray}
    \dot w_{1km} &=& \sum_i g_k[i] d_k[i] x_m[i] \\
    \dot w'_{1km} &=& \mathbb{I}(w_{1km} > 0) \sum_i g'_k[i] x_m[i]
\end{eqnarray}
Thanks to Lemma~\ref{lemma:relu-linear-identification}, we know the forward pass between two networks are identical and thus $g_k[i] = g'_k[i]$ so we don't need to consider the difference between backpropagated gradient.  

In the following, we will show that each summand of the two equations is identical.

\textbf{Case 1. $x_m[i] = 0$}. In that case, $g_k[i]x_m[i] = g_k[i]d_k[i]x_m[i] = 0$ regardless of whether the gate $d_k[i]$ is open or closed.  

\textbf{Case 2. $x_m[i] > 0$}. There are two subcases:

\emph{Subcase 1: $d_k[i] = 1$}. In this case, the ReLU gating of $k$-th filter is open, then $g_k'[i]x_m[i] = g_k[i]x_m[i] = g_k[i]d_k[i]x_m[i]$. By Assumption~\ref{assumption:orth-mixture}(One-hotness), for other $m'\neq m$, $x_{km'}[i] = 0$, since $d_k[i] = 1$, it must be the case that $w_{1km} > 0$ and thus $\mathbb{I}(w_{1km} > 0) = 1$. So the two summands are identical. 

\emph{Subcase 2: $d_k[i] = 0$}. Then $w_{1km}$ must be $0$, otherwise since $\vx_k \ge 0$ (nonnegativeness), we have $\vw_{1k}^\top \vx_k[i] \ge w_{1km} x_m[i] > 0$ and the gating of $k$-th filter must open. Therefore, the two summands are both $0$: the ReLU one is because $d_k[i] = 0$ and the linear one is due to $\mathbb{I}(w_{1km} > 0) = 0$.
\end{proof}

\onenodedynamics*
\begin{proof}
In \onehidden{}, since there is only one node, we have $X = \con_\alpha[\vx_1,\vx_1] = \con_\alpha[\vx,\vx]$. By Theorem~\ref{thm:property-of-relu}, the dynamics of $\vw_1$ is the linear dynamics plus the sticky weight rule, which is:
\begin{equation}
    \dot \vw_1 = \diag(\vw_1 > 0) X \vw_1  
\end{equation}

By Lemma~\ref{lemma:negative-do-not-matter}, the negative parts of $\vw_1$ can be removed without changing the result. Let's only consider the nonnegative part of $\vw$ and remove corresponding rows and columns of $X$. 

Note that the linear dynamics $\dot \vw_1 = X \vw_1$ will converge to certain maximal eigenvector $\vy$ (or its scaled version, depending on whether we have norm constraint or not). By Lemma~\ref{lemma:one-negative-entry}, as long as $X$ is not a scalar, $\vy$ has at least one negative entry. Therefore, by continuity of the trajectory of the linear dynamics, from $\vw_1$ to $\vy$, the trajectory must cross the boundary of the polytope $\vw_1 \ge 0$ that require all entries to be nonnegative. 

After that, according to the sticky weight rule, in the ReLU dynamics, the corresponding component (say $w_{1m}$) stays at zero. We can remove the corresponding $m$-th row and column of $X$, and the process repeats until $X$ becomes a scalar. Then $\vw_1$ converges to that remaining dimension. Since $\vw_1 \ge 0$, it must be the case that $\vw_1 \rightarrow \ve_m$ for some $m$.  
\end{proof}

\reludiversity*
\begin{proof}
We just need to prove that if the local optimal solution $(W_2, W_1)$ satisfies $\rank(W_1) = 1$, then $W_1 = \vv\ve^\top_m$ for some $m$ and $\vv \ge 0$.

Since $\rank(W_1) = 1$ and $\|W_1\|_F = 1$, by Lemma~\ref{lemma:rank-1-decomposition} we know that there exists unit vectors $\vu$ and $\vv$ so that $W_1 = \vv\vu^\top$. Since $W_1 \ge 0$, we can pick $\vu \ge 0$ and $\vv \ge 0$. Otherwise if $\vu$ has both positive and negative elements, then picking any nonzero element of $\vv$, the corresponding rows/colums of $W_1$ will also have both signs, which is a contradiction.  

Note that the objective function is 
\begin{equation}
 2\cE = \tr(W_2F_1 W_2^\top) = \tr(W_2W_1X_\alpha W_1^\top W_2^\top) = (\vu^\top X_\alpha \vu) \|W_2\vu\|_2^2 > 0 
\end{equation}
Therefore, $\vu^\top X_\alpha \vu > 0$ and $\|W_2\vu\|_2 > 0$. By Lemma~\ref{lemma:mm-property}, we know that if $W_2$ with the constraint $\|W_2\|_F = 1$ is an local optimal, $W_2$ is a rank-1 matrix with decomposition $W_2 = \vb\vv^\top$ with $\|\vb\|_2=1$. 

Then we have $2\cE = \vu^\top X_\alpha \vu > 0$ with $\vu \ge 0$. From the proof of Lemma~\ref{lemma:one-negative-entry}, we know that $X_\alpha$ has a unique \emph{minimal} all-positive eigenvector $\vc > 0$. 

If there are $\ge 2$ positive elements in $\vu$, then we can always create a vector $\va$ (with mixed signs in its elements) so that (1) $\va$ has the same non-zero support as $\vu$ and (2) $\va^\top \vc = 0$. Therefore, $\va$ is in the space of orthogonal complement of $\vc$. Since $\vc$ is the unique minimal eigenvector, moving $\vu$ along the direction of $\va$ will strictly improve $\cE$, which contradicts with the fact that $(W_2, W_1)$ is locally optimal. 

Therefore, the unit vector $\vu$ has only $1$ positive entry, which is $\ve_m$ for some $m$. Fig.~\ref{fig:relu-diversity} shows one example of learned weights with $\mathrm{rank} > 1$. 
\end{proof}

\section{More Experiments}
\label{sec:appendix-more-exp}
We also provide experiments with different batchsize (i.e., 256) and ablation studies on different exponent $p$ in the direct version of $\alpha$-CL. Note that we refer an unnormalized \method{}-direct as the following:
\begin{equation}
    \alpha_{ij} = \exp(-d^p_{ij}/\tau) \label{eq:unnormalized-direct-alpha}
\end{equation}
while (normalized) \method{}-direct as the following (same as Eqn.~\ref{eq:direct-alpha} in the main text):
\begin{equation}
     \alpha_{ij} = \frac{\exp(-d^p_{ij}/\tau)}{\sum_j\exp(-d^p_{ij}/\tau)}   \label{eq:normalized-direct-alpha} 
\end{equation}
By default, we set the exponent $p=4$ and $\tau = 0.5$.
\begin{table}[]
    \centering
    \begin{tabular}{l|l||c|c|c}
Dataset & Methods & 100 epochs & 300 epochs & 500 epochs \\
\hline\hline
\multirow{3}{*}{CIFAR-10} & $\cL_{nce}$ & $86.84 \pm 0.26$ & $89.19 \pm 0.15$ & $\mathbf{91.07 \pm 0.12}$ \\
& \method{}-direct (Eqn.~\ref{eq:unnormalized-direct-alpha}) &	$87.74 \pm 0.28$ & $89.76 \pm 0.26$ & $91.06 \pm 0.09$ \\
& \method{}-direct (Eqn.~\ref{eq:normalized-direct-alpha}) & $\mathbf{87.91\pm 0.12}$ & $\mathbf{89.89\pm 0.18}$ & $91.06\pm 0.17$ \\
\hline
\multirow{3}{*}{CIFAR-100} & $\cL_{nce}$	& $60.70 \pm 0.40$ & $64.22 \pm 0.19$ & $\mathbf{66.84 \pm 0.16}$ \\
&  \method{}-direct (Eqn.~\ref{eq:unnormalized-direct-alpha}) & $63.28 \pm 0.31$ & $65.71 \pm 0.20$ & $66.73 \pm 0.13$ \\
&  \method{}-direct (Eqn.~\ref{eq:normalized-direct-alpha}) & $\mathbf{63.47\pm 0.06}$ & $\mathbf{65.86\pm 0.24}$ & $66.57\pm 0.21$ \\
\hline
\multirow{3}{*}{STL10} &	$\cL_{nce}$ & $82.09 \pm 0.31$ & $86.96 \pm 0.19$ & $87.31 \pm 0.17$ \\
&  \method{}-direct (Eqn.~\ref{eq:unnormalized-direct-alpha}) & $83.00 \pm 0.28$ & $87.35 \pm 0.28$ & $87.63 \pm 0.29$ \\
&  \method{}-direct (Eqn.~\ref{eq:normalized-direct-alpha}) & $\mathbf{83.20\pm 0.17}$ & $\mathbf{87.36\pm 0.12}$ & $\mathbf{87.71\pm 0.14}$
    \end{tabular}
    \caption{\small Top-1 downstream task accuracy with ResNet50 backbone and 256 batchsize. Learning rate is 0.001. We also compare unnormalized \method{}-direct (Eqn.~\ref{eq:unnormalized-direct-alpha}) versus (normalized) \method{}-direct (Eqn.~\ref{eq:normalized-direct-alpha}). Normalized version, which is used in the main text of the paper, performs slightly better.}
    \label{tab:batchsize256}
\end{table}

\begin{table}[]
    \centering
    \setlength{\tabcolsep}{1pt}
    \begin{tabular}{l||c|c|c|c|c}
Exponent $p$ & $p = 2$ & $p = 4$ & $p = 6$ & $p = 8$ & $p = 10$  \\
\hline\hline
Top-1 accuracy (500 epochs) & $83.74 \pm 0.18$ & $84.06 \pm 0.24$	& $\mathbf{84.08 \pm 0.42}$ & $83.91 \pm 0.28$ &	$83.56 \pm 0.13$
    \end{tabular}
    \caption{\small Ablation study on different exponent $p$ in STL10 for the normalized pairwise importance (Eqn.~\ref{eq:normalized-direct-alpha}) in \method{}-direct.}
    \label{tab:p-ablation}
\end{table}

\section{Other Lemmas}
\begin{lemma}[Gradient Formula of contrastive Loss (Eqn.~\ref{eq:general-loss}) (extension of Lemma 2 in~\cite{jing2021understanding}]
\label{lemma:gradient-form}
Consider the loss function 
\begin{equation}
    \min_{\vtheta} \cL_{\phi,\psi}(\vtheta) := \sum_{i=1}^N \phi\left(\sum_{j\neq i} \psi(d^2_i - d^2_{ij})\right) 
\end{equation}
Then for any matrix (or vector) variable $A$, we have:
\begin{equation}
    \sum_{i=1}^N \frac{\partial \cL_{\phi,\psi}}{\partial \vz[i]} A^\top[i] + \frac{\partial \cL_{\phi,\psi}}{\partial \vz[i']} A^\top[i'] = -\con_\alpha[\vz, A]
\end{equation}
and 
\begin{equation}
    \sum_{i=1}^N A[i]\frac{\partial \cL_{\phi,\psi}}{\partial \vz[i]} + A[i']\frac{\partial \cL_{\phi,\psi}}{\partial \vz[i']} = -\con_\alpha[A, \vz^\top]
\end{equation}
where $\con_\alpha[\cdot,\cdot]$ is the \emph{contrastive covariance} defined as (here $\beta_i := \sum_{j\neq i}\alpha_{ij}$): 
\begin{equation}
  \con_\alpha[\vx,\vy] := \sum_{i,j=1}^N \alpha_{ij} (\vx[i] - \vx[j])(\vy[i]-\vy[j])^\top - \sum_{i=1}^N \beta_i (\vx[i] - \vx[i'])(\vy[i]-\vy[i'])^\top
\end{equation}
and $\alpha$ is defined as the following:
\begin{equation}
     \alpha_{ij} := \phi'\left(\sum_{j\neq i} \psi(d^2_i-d^2_{ij})\right) \psi'(d^2_i - d^2_{ij}) \ge 0
\end{equation}
where $\phi',\psi'$ are derivatives of $\phi,\psi$.
\end{lemma}
\begin{proof}
Taking derivative of the loss function $\cL = \cL_{\phi,\psi}$ w.r.t. $\vz[i]$ and $\vz[i']$, we have:
\begin{eqnarray}
    \frac{\partial \cL}{\partial \vz[i]} &=& \sum_{j\neq i} \alpha_{ij} (\vz[j] - \vz[i']) + \sum_{j\neq i} \alpha_{ji} (\vz[j] - \vz[i]) \label{eq:partial-z}\\
    \frac{\partial \cL}{\partial \vz[i']} &=& \sum_{j\neq i} \alpha_{ij} (\vz[i'] - \vz[i]) = \beta_i (\vz[i'] - \vz[i]) \label{eq:partial-z-prime} 
\end{eqnarray}

We just need to check the following:
\begin{equation}
   \sum_i
   \left(\sum_{j\neq i} \alpha_{ij} (\vz[j] - \vz[i']) + \sum_{j\neq i}\alpha_{ji}(\vz[j] - \vz[i])\right) A^\top[i] + \sum_i\beta_i (\vz[i'] - \vz[i]) A^\top [i']
\end{equation}

To see this, we only need to check whether the following is true:
\begin{equation}
-\Sigma_0 = \sum_i \left(\sum_{j\neq i} \alpha_{ij} (\vz[j] - \vz[i']) + \sum_{j\neq i}\alpha_{ji}(\vz[j] - \vz[i])\right)A^\top[i] + \sum_i \beta_i (\vz[i'] - \vz[i])A^\top[i] 
\end{equation}
which means that 
\begin{equation}
-\Sigma_0 = \sum_i \left(\sum_{j\neq i} \alpha_{ij} (\vz[j] - \vz[i]) + \sum_{j\neq i}\alpha_{ji}(\vz[j] - \vz[i])\right)A^\top[i] 
\end{equation}
Since $\alpha_{ii}(\vz[i] - \vz[i]) = 0$ for arbitrarily defined $\alpha_{ii}$, $j$ can also take the value of $i$, this leads to 
\begin{equation}
-\Sigma_0 = \sum_{i,j} \alpha_{ij} (\vz[j] - \vz[i])A^\top[i] + \sum_{i,j}\alpha_{ji}(\vz[j] - \vz[i])A^\top[i] 
\end{equation}
Swapping indices for the second term, we have:
\begin{eqnarray}
-\Sigma_0 &=& \sum_{i,j} \alpha_{ij} (\vz[j] - \vz[i])A^\top[i] + \sum_{i,j}\alpha_{ij}(\vz[i] - \vz[j])A^\top[j] \\
&=& \sum_{i,j} \alpha_{ij} (\vz[j] - \vz[i])A^\top[i] - \sum_{i,j}\alpha_{ij}(\vz[j] - \vz[i])A^\top[j] \\
&=& - \sum_{i,j} \alpha_{ij} (\vz[j] - \vz[i])(A^\top[j] - A^\top[i]) 
\end{eqnarray}
and the conclusion follows. 
\end{proof}

\begin{lemma}
\label{lemma:linear-entropy}
The following minimization problem: 
\begin{equation}
    \min_{p_j} \sum_j c_j p_j - \tau H(p)\quad \mathrm{s.t.} \sum_j p_j = \frac{1}{\tau}x_0\phi'(x_0) 
\end{equation}
where $H(p) := -\sum_j p_j \log p_j$ is the entropy and $x_0 :=\sum_{j} e^{-c_j/\tau}$, has close-form solution:
\begin{equation}
    p_j = \frac{1}{\tau}\exp(-c_j/\tau) \phi'\left(\sum_{j} \exp(-c_j/\tau)\right)
\end{equation}
\end{lemma}
\begin{proof}
Define the following Lagrangian multiplier:
\begin{equation}
    \cJ(\alpha, \vtheta) := \sum_j c_j p_j - \tau H(p) + \mu \left(\sum_j p_j - \frac{1}{\tau}x_0\phi'(x_0)\right)
\end{equation}
Taking derivative w.r.t $p_j$ and we have: 
\begin{equation}
\frac{\partial \cJ}{\partial p_j} = c_j + \tau(\log p_j + 1) - \mu = 0
\end{equation}
which gives the solution
\begin{equation}
    p_j = \exp\left(\frac{\mu}{\tau} - 1\right) \exp\left(-\frac{c_j}{\tau}\right) := Z \exp\left(-\frac{c_j}{\tau}\right) 
\end{equation}
where $Z$ can be computed via the constraint: 
\begin{equation}
    Z = \frac{1}{\tau}\frac{x_0\phi'(x_0)}{\sum_j e^{-c_j/\tau}} = \frac{1}{\tau} \phi'(x_0)
\end{equation}
\end{proof}

\begin{lemma}
\label{lemma:normalization-lemma}
The normalization function $\vy = (\vx - \mathrm{mean}(\vx)) / \|\vx\|_2$ has the following forward/backward rule:
\begin{equation}
    \vy = J(\vx)\vx, \quad \frac{\partial \vy}{\partial \vx} = J^\top(\vx) 
\end{equation}
where $J(\vx) := \frac{1}{\|P^\perp_{\vone} \vx\|_2} P^\perp_{\vx,\vone}$ is a symmetric matrix. For $\vy = \vx / \|\vx\|_2$, the relationship still holds with $J(\vx) = \frac{1}{\|\vx\|_2} P^\perp_{\vx}$. 
\end{lemma}
\begin{proof}
See Theorem 5 in~\cite{tian2018theoretical}.
\end{proof}

\begin{lemma}
\label{lemma:layernorml2}
Suppose the output of a linear layer (with a weight matrix $W_l$) connects to a $\ell_2$ regularization or LayerNorm through reversible layers, then $\frac{\dd}{\dd t} \|W_l\|^2_F = 0$.
\end{lemma}
\begin{proof}
From Lemma, for each sample $i$, we have its gradient before/after the normalization layer (say it is layer $m$) to be the following: 
\begin{equation}
    \vg_m[i] = {J_m^\bn[i]}^\top\vg_m^\bn[i]
\end{equation}
where $\vg_m[i]$ is the gradient after back-propagating through normalization, and $\vg_m^\bn[i]$ is the gradient sending from the top level. 

Here $J_m^\bn[i] = \frac{1}{\|P^\perp_{\vone}\vf_m[i]\|_2} P^\perp_{\vf_m[i], \vone}$ for LayerNorm and $J_m^\bn[i] = \frac{1}{\|\vf_m[i]\|_2} P^\perp_{\vf_m[i]}$ for $\ell_2$ normalization. For $W_l$, its gradient update rule is:
\begin{equation}
    \dot W_l = \sum_i \tilde\vg_l[i] \vf^\top_{l-1}[i]
\end{equation}
By reversibility, we know that $\tilde\vg_l[i] = J_{(\tilde l, m]}^\top[i] \vg[i]$, where $J_{(\tilde l, m]}[i]$ is the Jacobian after the linear layer $\tilde l$ till layer $m$, right before the normalization layer. Therefore, we have:
\begin{eqnarray}
    \tr(W_l^\top \dot W_l) &=& \sum_i \tr(W_l^\top J_{(\tilde l, m]}^\top[i]{J_m^\bn[i]}^\top \vg_m^\bn[i] \vf^\top_{l-1}[i]) \\
    &=& \sum_i \tr(\vf^\top_{l-1}[i]W_l^\top J_{(\tilde l, m]}^\top[i]{J_m^\bn[i]}^\top\vg_m^\bn[i]) \\
    &=& \sum_i \tr(\vf_m^\top[i]{J_m^\bn[i]}^\top\vg_m^\bn[i]) \\
    &=& 0
\end{eqnarray}
The last two equality is due to reversibility $\vf_m[i] = J_{(\tilde l, m]}[i] W_l\vf_{l-1}[i]$ and the property of normalization layers: $J_m^\bn[i] \vf_m[i] = 0$, since a vector projected to its own complementary space is always zero $P^\perp_{\vf_m[i]}\vf_m[i] = 0$.  

Then we have 
\begin{equation}
\frac{\dd}{\dd t} \|W_l\|^2_F = \frac{\dd}{\dd t} \tr(W_l^\top W_l) =  \tr(\dot W_l^\top W_l) +  \tr(W_l^\top \dot W_l) = 0
\end{equation}
\end{proof}

\begin{lemma}
\label{lemma:rank-1-decomposition}
For every rank-1 matrix A with $\|A\|_F = 1$, there exists $\|\vu\|_2 = \|\vv\|_2 = 1$ so that $A = \vu\vv^\top$. 
\end{lemma}
\begin{proof}
Since $A$ is rank-1, it is clear that there exists $\vu'$ and $\vv'$ so that $A = \vu'{\vv'}^\top$. Since $\|A\|_F = 1$, we have $\|A\|^2_F:=\tr(AA^\top) = \|\vu'\|^2_2 \|\vv'\|^2_2 = 1$. Therefore, taking $\vu = \vu' / \|\vu'\|_2$ and $\vv = \vv' / \|\vv'\|_2$, we have $A = \vu\vv^\top$.  
\end{proof}

\begin{lemma}
\label{lemma:norm2rank1}
If $\|W_l\|_F = 1$ for $1 \le l \le L$, then $\|W_L W_{L-1} \ldots W_1\|_2 = 1$ if any only if $W_L, W_{L-1}, \ldots, W_1$ are aligned-rank-1 (Def.~\ref{def:rank-1}).
\end{lemma}
\begin{proof}
If $W_L, W_{L-1}, \ldots, W_1$ are aligned-rank-1, then by its definition, there exists unit vectors $\{\vv_l\}_{l=0}^L$ so that $W_l = \vv_l\vv_{l-1}^\top$. Therefore, $\|W_L W_{L-1} \ldots W_1\|^2_2 = \|\vv_L\vv_0^\top\|^2_2 = \lambda_{\max}(\vv_L\vv_0^\top\vv_0\vv_L^\top) = \lambda_{\max}(\vv_L\vv_L^\top) = 1$.

Then we prove the other direction. Note that
\begin{equation}
    \|W_L W_{L-1} \ldots W_1\|_2 \le \prod_{l=1}^L \|W_l\|_2 \le \prod_{l=1}^L \|W_l\|_F = 1 \label{eq:w-upperbound}
\end{equation}
and the equality only holds when all $W_l$ are rank-1. By Lemma~\ref{lemma:rank-1-decomposition}, for any $l$, there exists unit vectors $\vv'_l$, $\vv_{l-1}$ so that $W_l = \vv'_l \vv_{l-1}^\top$. To show that they must be aligned (i.e. $\vv_l = \pm\vv'_l$), we prove by contradiction.  

Suppose $\|W_L W_{L-1} \ldots W_1\|_2 = 1$ but for some $l$, $\vv'_l \neq \pm\vv_l$ and thus $|\vv^\top_l\vv'_l| < 1$. Then $W_{l+1}W_l = (\vv^\top_l\vv'_l) \vv_{l+1}\vv^\top_{l-1}$ and $\|W_{l+1}W_l\|_2 \le \|W_{l+1}W_l\|_F = |\vv^\top_l\vv'_l| < 1$. Therefore, $\|W_L W_{L-1} \ldots W_1\|_2 < 1$, which is a contradiction.   

Note that for $W_l = \pm \vv_l\vv_{l-1}^\top$, we can always move around the signs to either $\vv_0$ or $\vv_L$ to fit into the definition of aligned-rank-1. 
\end{proof}

\begin{lemma}
\label{lemma:mm-property}
For the following optimization problem with a given fixed vector $\vu\neq 0$:
\begin{equation}
    \max_{\cW} \cJ(\cW;\vu) := \|W_L W_{L-1} \ldots W_1\vu\|_2\quad \mathrm{s.t.\ } \|W_l\|_F = 1, 
\end{equation}
where $\cW = \{W_L, W_{L-1}, \ldots, W_1\}$. If $\cW^*$ is a local maximum solution (i.e., there exists a neighborhood $\cN(\cW^*)$ of $\cW^*$ so that for any $\cW \in \cN(\cW^*)$, $\cJ(\cW) \le \cJ(\cW^*)$), and $\cJ(\cW^*) > 0$, then $\cW^*$ is an aligned-rank-1 solution (Def.~\ref{def:rank-1}).
\end{lemma}
\begin{proof}
Let $\vv'_{L-1} := W^*_{L-1} W^*_{L-2} \ldots W^*_1\vu$. Note that $\vv'_{L-1} \neq 0$ (otherwise $\cJ(\cW^*)$ would be zero). Consider the following optimization subproblem (here we optimize over $W_L$ and treat $\vv'_{L-1}$ as a fixed vector). 
\begin{equation}
    \max_{W_L} \cJ(W_L;W^*_{-L}) = \|W_L \vv'_{L-1}\|_2\quad \mathrm{s.t.\ } \|W_L\|_F = 1 \label{eq:subproblem} 
\end{equation}
By local optimality of $\cW^*$, $W^*_L$ must be the local maximum of Eqn.~\ref{eq:subproblem} and thus a critical point, since both the objective and the constraints are differentiable. Note that $\|W_L \vv'_{L-1}\|_2$ is a vector 2-norm and all critical points of Eqn.~\ref{eq:subproblem} must satisfy
\begin{equation}
    W_L \vv'_{L-1} {\vv'}^\top_{L-1} = \lambda W_L 
\end{equation}
for some constant $\lambda$. Notice that to satisfy this condition, each row of $W_L$ must be an eigenvector of $\vv'_{L-1} {\vv'}^\top_{L-1}$. For a solution to be local maximal, $\lambda$ is the largest eigenvalue of $\vv'_{L-1} {\vv'}^\top_{L-1}$, and each row of $W_L$ is the corresponding eigenvector. It is clear that the rank-1 matrix $\vv'_{L-1} {\vv'}^\top_{L-1}$ has a unique maximum eigenvalue $\|\vv'_{L-1}\|_2^2 > 0$ with its corresponding one-dimensional eigenspace span by $\vv_{L-1} :=\vv'_{L-1} / \|\vv'_{L-1}\|_2$ (while all other eigenvalues are zeros). Therefore, $W^*_L$ as the local maximum of Eqn.~\ref{eq:subproblem}, must have:
\begin{equation}
    W^*_L = \vv_L \vv^\top_{L-1} \label{eq:wL}
\end{equation}
for some $\|\vv_L\|_2 = 1$. 

Now let $\vv'_{L-2} := W^*_{L-2} \ldots W^*_1\vu$. Similarly, $\vv'_{L-2} \neq 0$ (otherwise $\cJ(\cW^*)$ would be zero). Then $\vv'_{L-1} = W^*_{L-1} \vv'_{L-2}$. Treating $\vv'_{L-2}$ as a fixed vector and varying $W_{L-1}$ and $W_L$ simultaneously, then since $W^*_{L:1}$ is a local maximal solution, $W^*_L$ must take the form of Eqn.~\ref{eq:wL} given any $W^*_{L-1}$, which means that the objective function now becomes 
\begin{equation}
    \cJ(W_{L-1};W^*_{-(L-1)}) = \|W^*_L\vv'_{L-1}\|_2 = \|\vv_L \vv^\top_{L-1}\vv'_{L-1}\|_2 = \|\vv'_{L-1}\|_2 = \|W_{L-1}\vv'_{L-2}\|_2
\end{equation}
and the subproblem becomes: 
\begin{equation}
    \max_{W_{L-1}} \|W_{L-1} \vv'_{L-2}\|_2\quad \mathrm{s.t.\ } \|W_{L-1}\|_F = 1 \label{eq:subproblem-2} 
\end{equation}

Repeating this process, we know $W^*_{L-1}$ must satisfy: 
\begin{equation}
    W^*_{L-1} = \vv_{L-1} \vv^\top_{L-2} \label{eq:wL-2}
\end{equation}
for $\vv_{L-2} := \vv'_{L-2} / \|\vv'_{L-2}\|_2$. This procedure can be repeated until $W_1$ and the prove is complete.  

\end{proof}

\begin{restatable}{lemma}{lemmalbn}
\label{lemma:bn-weight}
$\frac{\dd}{\dd t} \|\vw_{k}\|^2_2 = 0$, if node $k$ is under BatchNorm.
\end{restatable}
\begin{proof}
For BN, it is a layer with reversibility on each filter $k$. We use $\vf_k,\vg_k \in \rr^{N}$ to represent the activation/gradient at node $k$ in a batch of size $N$. The forward/backward operation of BN can be written as: 
\begin{equation}
    \vf^\bn_k = J_k \vf_k, \quad \vg_k = J_k^\top \vg^\bn_k
\end{equation}
Here $J_k = J_k^\top = \frac{1}{\|P^\perp_{\vone} \vf_k \|_2} P^\perp_{\vf_k,\vone}$ is the Jacobian matrix at each node $k$.

We check how the weight $\vw_k$ changes under BatchNorm. Here we have $\vf_k = h(F_{l-1}\vw_k)$ where $h$ is a reversible activation and $F_{l-1} \in \rr^{N\times n_{l-1}}$ contains all output from the last layer. Then we have:
\begin{equation}
    \dot \vw_k = \sum_i h'_i g_k[i] \vf_{l-1}[i] = F^\top_{l-1} D_k \vg_k = F^\top_{l-1} D_k J^\top_k \vg^\bn_k 
\end{equation}
where $D_k := \diag([h'_i]_{i=1}^N) \in \rr^{N\times N}$. Due to reversibility, we have $\vf_k = h(F_{l-1}\vw_k) = D_k F_{l-1}\vw_k$. Therefore,
\begin{equation}
    \vw_k^\top \dot \vw_k = \vw_k^\top F^\top_{l-1} D_k J^\top_k \vg^\bn_k = \vf_k^\top J^\top_k \vg^\bn_k = 0 
\end{equation}
\end{proof}

\begin{lemma}[BatchNorm regularization]
\label{lemma:bn-reg}
Consider the following optimization problem with a fixed vector $\vu \neq 0$:
\begin{equation}
    \max_{\cW} \cJ(\cW) := \|W_L W_{L-1} \ldots W_1\vu\|_2\quad \mathrm{s.t.\ } \|W_L\|_F = 1,\quad \|\vw_{lk}\|_2 = 1/\sqrt{n_l} 
\end{equation}
where $\cW := \{W_L, W_{L-1},\ldots, W_1\}$ and $\vw_{lk}$ are rows of $W_l$ (i.e., weight of the $k$-th filter at layer $l$). Then Lemma~\ref{lemma:mm-property} still holds by replacing aligned-ranked-one with aligned-uniform condition.  
\end{lemma}
\begin{proof}
The proof is basically the same. The only difference here is that the sub-problem (Eqn.~\ref{eq:subproblem-2}) becomes:
\begin{equation}
\max_{W_l} \|W_l \vv'_{l-1}\|_2\quad \mathrm{s.t.\ } \|\vw_{lk}\|_2 = 1 / \sqrt{n_l}
\end{equation}
for $1 \le l \le L-1$. The critical point condition now becomes (here $\Lambda$ is a diagonal matrix):
\begin{equation}
    W_l \vv'_{l-1} {\vv'}^\top_{l-1} = \Lambda W_l
\end{equation}
That is, each row of $W_l$ now has a different constant. Since the eigenvalue of $\vv'_{l-1} {\vv'}^\top_{l-1}$ can only be 0 or 1, and $0$ won't work (otherwise the corresponding row of $W_l$ would be a zero vector, violating the row-norm constraint), all diagonal element of $\lambda$ has to be 1. Therefore, $W_l = \vv_{l}\vv_{l-1}^\top$. Due to row-normalization, we have $[\vv_{l}]_k = \pm 1 / \sqrt{n_l}$ for $1 \le l \le L-1$, while $\vv_L$ and $\vv_0$ can still take arbitrary unit vector.  
\end{proof}

\begin{lemma}
\label{lemma:relu-linear-identification}
If Assumption~\ref{assumption:orth-mixture}(Nonnegativeness) holds, then a 2-layer ReLU network with weights $\vw_{1k} \ge 0$ and $W_2$ has the same activations (i.e., $\vf_l = \vf_l'$) as its linear network counterpart with the same weights $\vw'_{1k} = \vw_{1k}$ and $W_2' = W_2$.  
\end{lemma}
\begin{proof}
Since $W_2' = W_2$, we only need to prove $\vf_1 = \vf'_1$. For each filter $k$, we have its activation $f_{1k} = \max(\sum_m w_{1km} x_{km}, 0)$ and $f'_{1k} = \sum_m w'_{1km} x_{km} = \sum_m w_{1km} x_{km}$. By Assumption~\ref{assumption:orth-mixture}(nonnegativeness), all $x_{km} \ge 0$. Since $w_{1km} \ge 0$, $\sum_m w_{1km} x_{km} \ge 0$ and $f_{1k} = f'_{1k}$. 
\end{proof}

\begin{restatable}{lemma}{onenegativeentry}
\label{lemma:one-negative-entry}
If Assumption~\ref{assumption:orth-mixture} holds, $M \ge 2$, $\vx_1$ covers all $M$ modes, and $\alpha_{ij} > 0$, then the maximal eigenvector of $X_\alpha$ always contains at least one negative entry.  
\end{restatable}
\begin{proof}
Let $X_k := \con_\alpha[\vx_k,\vx_k]$. By Lemma~\ref{lemma:off-diagonal-negative}, all off-diagonal elements of $X_k$ are negative. Then $X_k$ can be written as $X_k = \beta I - X'_k$ for some $\beta$ where $X'_k$ is a symmetric matrix whose entries are all positive. By Perron–Frobenius theorem, $X'_k$ has a unique maximal eigenvector $\vu_k > 0$ (with all positive entries) and its associated positive eigenvalue $\lambda_k > 0$. Therefore, $\vu_k > 0$ is also the unique(!) minimal eigenvector of $X_k$. Since $M \ge 2$, there exists a maximal eigenspace, in which any maximal eigenvector $\vy_k$ satisfies $\vy_k^\top\vu_k = 0$. By Lemma~\ref{lemma:negative-entry}, the theorem holds. 
\end{proof}

\begin{lemma}
\label{lemma:off-diagonal-negative}
If the receptive field $R_k$ satisfies Assumption~\ref{assumption:orth-mixture}, and the collection of $N$ vectors $\{\vx_k[i]\}_{i=1}^N$ contains all $M$ modes, then all off-diagonal elements of $\con_\alpha[\vx_k,\vx_k]$ are negative.  
\end{lemma}
\begin{proof}
We check every entry of $X_k := \con_\alpha[\vx_k,\vx_k]$. Let $\beta_i := \sum_{j\neq i}\alpha_{ij}$. Note that for off-diagnoal element $[X_k]_{ml}$ with $m \neq l$, we have:
\begin{equation}
    [X_k]_{ml} = \sum_{ij}\alpha_{ij} (x_{km}[i] - x_{km}[j])(x_{kl}[i] - x_{kl}[j]) - \sum_{i} \beta_i (x_{km}[i] - x_{km}[i'])(x_{kl}[i] - x_{kl}[i'])
\end{equation}
Let $A_m := \{i: x_{km}[i] > 0\}$ be the sample set in which the $m$-th component is strictly positive, and $A_m^c := \{1, 2, \ldots, N\} \backslash A_m$ its complement. By Assumption~\ref{assumption:orth-mixture}(one-hotness), if $i\in A_m$ then $i\in A^c_{m'}$ for any $m' \neq m$. 

Now we consider several cases for sample $i$ and $j$:

\textbf{Case 1, $i,j\in A_m$}. Then $i,j \in A^c_l$ for $l\neq m$. This means that $x_{kl}[i] - x_{kl}[j] = 0$. 

\textbf{Case 2, $i,j\in A^c_m$}. Then $x_{km}[i] - x_{km}[j] = 0$.

\textbf{Case 3, $i\in A_m$ and $j\in A^c_m$}. Since $j\in A^c_m$, we have $x_{km}[i] - x_{km}[j] = x_{km}[i] > 0$. On the other hand, since $i\in A_m$, $i \in A_l^c$, we have $x_{kl}[i] - x_{kl}[j] = - x_{kl}[j] \le 0$. Therefore, $(x_{km}[i] - x_{km}[j])(x_{kl}[i] - x_{kl}[j]) \le 0$. 

\textbf{Case 4. $i\in A^c_m$ and $j\in A_m$}. This is similar to Case 3. 

Putting them all together, since $\alpha_{ij} > 0$, we know that 
\begin{equation}
    \sum_{ij}\alpha_{ij} (x_{km}[i] - x_{km}[j])(x_{kl}[i] - x_{kl}[j]) \le 0
\end{equation}
Furthermore, it is strictly negative since for $i \in A_m$ and $j\in A_l$, we have
\begin{equation}
    (x_{km}[i] - x_{km}[j])(x_{kl}[i] - x_{kl}[j]) = - x_{km}[i]x_{kl}[j] < 0
\end{equation}
By our assumption that the $N$ vectors $\{\vx_k[i]\}_{i=1}^N$ contains all $M$ modes, both $A_m$ and $A_l$ are not empty so this is achievable.  

For the second summation, by Assumption~\ref{assumption:orth-mixture}(Augmentation), either $i,i'\in A_m$ or $i,i'\in A_m^c$, it is always zero for $m\neq l$.  
\end{proof}

\begin{lemma}
\label{lemma:negative-entry}
If $\vv > 0$ is an all positive $d$-dimensional vector, $\vu^\top\vv = 0$, then 
\begin{equation}
    \min_m u_m \le -\frac{\min_m v_m}{d-k} \frac{\|\vu\|_\infty}{\|\vv\|_\infty}
\end{equation}
where $k$ is the number of nonnegative entries in $\vu$. 
\end{lemma}
\begin{proof}
Let $m_0 := \argmax_m |u_m|$. If $u_{m_0} = -\|\vu\|_\infty = \min_m u_m$ then we have proven the theorem. Otherwise $u_0 := u_{m_0} \ge 0$. $u_{m_0}$ is the largest entry of $\{u_m\}$. 

Since $\min_m u_m < 0$, by Rearrangement inequality we have:
\begin{equation}
    0 = \vu^\top \vv = \sum_m u_m v_m \ge \left(\min_m v_m\right) u_0 + (d-k) \left(\max_m v_m\right)\left(\min_m u_m\right)
\end{equation}
The conclusion follows. 
\end{proof}

\end{document}